\documentclass{article}

\usepackage[accepted]{icml2019}

\usepackage{microtype}
\usepackage{graphicx}
\usepackage{subfigure}
\usepackage{booktabs} % for professional tables
\usepackage{url}
\usepackage{mathtools}
\mathtoolsset{
showonlyrefs=true % or false in draft mode
}
\usepackage{amsmath, amsthm, bbm, amsfonts, amssymb, cite, enumerate, ctable}
\usepackage{cancel}
\usepackage{datetime, comment}
\usepackage{algorithmic}
\usepackage{algorithm}
\usepackage[normalem]{ulem}

\DeclareMathOperator*{\oracle}{\textsf{oracle}}
\DeclareMathOperator*{\rank}{rank}
\DeclareMathOperator*{\im}{im}

\DeclareMathOperator*{\hull}{hull}

\DeclareMathOperator*{\argmax}{argmax}

\DeclareMathOperator*{\ggrad}{grad}
\DeclareMathOperator*{\dv}{div}
\DeclareMathOperator*{\curl}{curl}

\DeclareMathOperator*{\expec}{\mathbb E}

\newcounter{todoCounter}

\newcommand*\perf{\ensuremath{v}}
\newcommand*\diverse{\ensuremath{d}}

\newcommand*\bOmega{\ensuremath{\boldsymbol\Omega}}

\newcommand*\bpi{\ensuremath{\boldsymbol\pi}}
\newcommand*\bxi{\ensuremath{\boldsymbol\xi}}
\newcommand*\balpha{\ensuremath{\boldsymbol\alpha}}

\newcommand*\grad{\nabla}

\DeclarePairedDelimiter\floor{\lfloor}{\rfloor}

\newcounter{dbaCounter}

\newcounter{garCounter}

\newcounter{wojCounter}

\newcounter{maxCounter}

\newcommand{\paper}{{\mathfrak p}}
\newcommand{\rock}{{\mathfrak r}}
\newcommand{\scissors}{{\mathfrak s}}
\newcommand{\three}{{\mathfrak T}}

\newcommand{\egs}{{\mathsf{EGS}}}
\newcommand{\fgs}{{\mathsf{FGS}}}
\newcommand{\ffg}{{\mathsf{FFG}}}
\newcommand{\uresp}{{\mathsf{PSRO_U}}}
\newcommand{\psro}{{\mathsf{PSRO_N}}}
\newcommand{\tfold}{{\mathsf{PSRO_{rN}}}}
\newcommand{\psr}{{\mathsf{PSRO}}}

\newcommand{\wt}{{\mathbf w}}
\newcommand{\bJ}{{\mathbf J}}
\newcommand{\bG}{{\mathbf G}}
\newcommand{\bM}{{\mathbf M}}

\newcommand{\br}{{\mathbf r}}

\newcommand{\ft}{{\mathbf f}}
\newcommand{\vt}{{\mathbf v}}
\newcommand{\x}{{\mathbf x}}

\newcommand{\bq}{{\mathbf q}}
\newcommand{\bp}{{\mathbf p}}

\newcommand{\be}{{\mathbf e}}

\newcommand{\bu}{{\mathbf u}}

\newcommand{\bO}{{\mathbf 1}}
\newcommand{\bZ}{{\mathbf 0}}
\newcommand{\bR}{{\mathbb R}}

\newcommand{\pop}{{\mathfrak P}}
\newcommand{\popq}{{\mathfrak Q}}
\newcommand{\popr}{{\mathfrak R}}

\newcommand{\bI}{{\mathbf I}}

\newcommand{\cC}{{\mathcal C}}
\newcommand{\cF}{{\mathcal F}}
\newcommand{\cG}{{\mathcal G}}
\newcommand{\bA}{{\mathbf A}}
\newcommand{\bB}{{\mathbf B}}

\newcommand{\Wt}{{\mathbf W}}

\newcommand{\bc}{{\mathbf c}}

\newcommand{\cP}{{\mathcal P}}
\newcommand{\cQ}{{\mathcal Q}}
\newcommand{\cN}{{\mathcal N}}

\newtheorem{thm}{Theorem}

\newtheorem{prop}[thm]{Proposition}

\newtheorem{lem}[thm]{Lemma}
\newtheorem{defn}{Definition}

\newtheorem{eg}{Example}

\newtheorem{thm*}{Theorem}
\newtheorem{prop*}[thm*]{Proposition}
\newtheorem{lem*}[thm*]{Lemma}

\icmltitlerunning{Open-ended Learning in Symmetric Zero-sum Games}

\begin{document}
\twocolumn[
	\icmltitle{Open-ended Learning in Symmetric Zero-sum Games}
	\begin{icmlauthorlist}
		\icmlauthor{David Balduzzi}{dm}
		\icmlauthor{Marta Garnelo}{dm}
		\icmlauthor{Yoram Bachrach}{dm}
		\icmlauthor{Wojciech M. Czarnecki}{dm}
		\icmlauthor{Julien Perolat}{dm}
		\icmlauthor{Max Jaderberg}{dm}
		\icmlauthor{Thore Graepel}{dm}
	\end{icmlauthorlist}
	\icmlaffiliation{dm}{DeepMind}
	\icmlcorrespondingauthor{}{dbalduzzi@google.com}
	\icmlkeywords{optimization, game theory, multi-agent learning}
	\vskip 0.3in
]

\printAffiliationsAndNotice

\begin{abstract}
    Zero-sum games such as chess and poker are, abstractly, functions that evaluate pairs of agents, for example labeling them `winner' and `loser'. If the game is approximately transitive, then self-play generates sequences of agents of increasing strength. However, nontransitive games, such as rock-paper-scissors, can exhibit strategic cycles, and there is no longer a clear objective -- we want agents to increase in strength, but against whom is unclear. In this paper, we introduce a geometric framework for formulating agent objectives in zero-sum games, in order to construct adaptive sequences of objectives that yield open-ended learning. The framework allows us to reason about population performance in nontransitive games, and enables the development of a new algorithm (rectified Nash response, $\tfold$) that uses game-theoretic niching to construct diverse populations of effective agents, producing a stronger set of  agents than existing algorithms. We apply $\tfold$ to two highly nontransitive resource allocation games and find that $\tfold$ consistently outperforms the existing alternatives.
\end{abstract}

%!TEX root = main.tex 
%%%%%%%%%%%%%%%%%%%%%%%%%%%%%%%
\section{Introduction}

A story goes that a Cambridge tutor in the mid-$19^\text{th}$ century once proclaimed: ``I'm teaching the smartest boy in Britain.'' His colleague retorted: ``I'm teaching the best test-taker.'' Depending on the version of the story, the first boy was either Lord Kelvin or James Clerk Maxwell. The second boy, who indeed scored highest on the Tripos, is long forgotten. 

Modern learning algorithms are outstanding test-takers: once a problem is packaged into a suitable objective, deep (reinforcement) learning algorithms often find a good solution. However, in many multi-agent domains, the question of what test to take, or what objective to optimize, is not clear. This paper proposes algorithms that adaptively and continually pose new, useful objectives which result in open-ended learning in two-player zero-sum games. This setting has a large scope of applications and is general enough to include function optimization as a special case. 

Learning in games is often conservatively formulated as training agents that tie or beat, on average, a fixed set of opponents. However, the dual task, that of \emph{generating useful opponents to train and evaluate against}, is under-studied. It is not enough to beat the agents you know; it is also important to generate better opponents, which exhibit behaviours that you don't know. 

There are very successful examples of algorithms that pose and solve a series of increasingly difficult problems for themselves through forms of self-play~\citep{silver:18, jaderberg2018human, bansal2017emergent, tesauro1995temporal}. Unfortunately, it is easy to encounter nontransitive games where self-play cycles through agents without improving overall agent strength -- simultaneously improving against one opponent and worsening against another. In this paper, we develop a mathematical framework for analyzing nontransitive games, and present algorithms that systematically uncover and solve the latent problems embedded in a game.

%%%%%%%%%%%%%%%%%%%%%%%%%%%%%%%
\textbf{Overview.}
The paper starts in Section~\ref{s:ffgs} by introducing functional-form games ($\ffg$s) as a new mathematical model of zero-sum games played by parametrized agents such as neural networks. Theorem~\ref{thm:hodge} decomposes any $\ffg$ into a sum of transitive and cyclic components. Transitive games, and closely related monotonic games, are the natural setting for self-play, but the cyclic components, present in non-transitive games, require more sophisticated algorithms which motivates the remainder of the paper. 

The main problem in tackling nontransitive games, where there is not necessarily a best agent, is understanding what the objective should be. In Section~\ref{s:gamescapes}, we formulate the global objective in terms of {\bf gamescapes} -- convex polytopes that encode the interactions between agents in a game. If the game is transitive or monotonic, then the gamescape degenerates to a one-dimensional landscape. In nontransitive games, the gamescape can be high-dimensional because training against one agent can be fundamentally different from training against another.

Measuring the performance of individual agents is vexed in nontransitive games. Therefore, in Section~\ref{s:gamescapes}, we develop tools to analyze populations of agents, including a population-level measure of performance, definition~\ref{def:perf}. An important property of population-level performance is that it increases transitively as the gamescape polytope expands in a nontransitive game. Thus, we reformulate the problem of learning in games from finding the best agent to growing the gamescape. We consider two approaches to do so, one directly performance related, and the other focusing on a measure of diversity, definition~\ref{def:diverse}. Crucially, the measure quantifies diverse \emph{effective behaviors} -- we are not interested in differences in policies that do not lead to differences in outcomes, nor in agents that lose in new and surprising ways.

Section~\ref{l_sect_alg} presents two algorithms, one old and one new, for growing the gamescape. The algorithms can be seen as specializations of the policy space response oracle ($\psr$) introduced in \citet{lanctot:17}. The first algorithm is Nash response ($\psro$), which is an extension to functional-form games of the double oracle algorithm from \citet{mcmahan:03}. Given a population, Nash response creates an objective to train against by averaging over the Nash equilibrium. The Nash serves as a proxy for the notion of `best agent', which is not guaranteed to exist in general zero-sum games. A second, complementary algorithm is the rectified Nash response ($\tfold$). The algorithm amplifies strategic diversity in populations of agents by adaptively constructing \emph{game-theoretic niches} that encourage agents to `play to their strengths and ignore their weaknesses'.

Finally, in Section~\ref{s:exp}, we  investigate the performance of these algorithms in Colonel Blotto \citep{borel:21, tukey:49, roberson:06} and a differentiable analog we refer to as differentiable Lotto. Blotto-style games involve allocating limited resources, and are highly nontransitive. We find that $\tfold$ outperforms $\psro$, both of which greatly outperform self-play in these domains. We also compare against an algorithm that responds to the uniform distribution $\uresp$, which performs comparably to $\psro$.

%%%%%%%%%%%%%%%%%%%%%%%%%%%%%%%
\textbf{Related work.}
There is a large literature on novelty search, open-ended evolution, and curiosity, which aim to continually expand the frontiers of game knowledge within an agent \citep{lehman:08, taylor:16, banzhaf:16, brant:17, pathak:17, wang:19}. A common thread is that of \emph{adaptive objectives} which force agents to keep improving. For example, in novelty search, the target objective constantly changes -- and so cannot be reduced to a fixed objective to be optimized once-and-for-all.

We draw heavily on prior work on learning in games, especially \citet{heinrich:15, lanctot:17} which are discussed below. Our setting resembles multiobjective optimization \citep{fonseca:93, miettinen:98}. However, unlike multiobjective optimization, we are concerned with \emph{both generating and optimizing} objectives. Generative adversarial networks \citep{goodfellow:14} are zero-sum games that do \emph{not} fall under the scope of this paper due to lack of symmetry, see appendix~\ref{s:gans}.

%%%%%%%%%%%%%%%%%%%%%%%%%%%%%%%
\textbf{Notation.}
Vectors are columns. The constant vectors of zeros and ones are $\bZ$ and $\bO$. We sometimes use $\bp[i]$ to denote the $i^\text{th}$ entry of vector $\bp$. Proofs are in the appendix. 

%!TEX root = main.tex
%%%%%%%%%%%%%%%%%%%%%%%%%%%%%%%
\section{Functional-form games ($\ffg$s)}
\label{s:ffgs}

Suppose that, given any pair of agents, we can compute the probability of one beating the other in a game such as Go, Chess, or StarCraft. We formalize the setup as follows.
\begin{defn}
    Let $W$ be a set of agents parametrized by, say, the weights of a neural net. A \textbf{symmetric zero-sum functional-form game} ($\ffg$) is an \emph{antisymmetric} function, $\phi(\vt,\wt) =- \phi(\wt,\vt)$, that evaluates pairs of agents
    \begin{equation}
        \phi : W\times W\rightarrow \bR.
    \end{equation}
    The higher $\phi(\vt, \wt)$, the better for agent $\vt$. We refer to $\phi>0$, $\phi<0$, and $\phi=0$ as wins, losses and ties for $\vt$.
\end{defn}
Note that (i) the strategies in a $\ffg$ are \emph{parametrized agents} and (ii) the parametrization of the agents is folded into $\phi$, so the game is a composite of the agent's architecture and the environment itself. 

Suppose the probability of $\vt$ beating $\wt$, denoted $P(\vt\succ\wt)$ can be computed or estimated. Win/loss probabilities can be rendered into antisymmetric form via $\phi(\vt,\wt) := P(\vt\succ\wt)-\frac{1}{2}$ or $\phi(\vt,\wt) := \log\frac{P(\vt\succ\wt)}{P(\vt\prec\wt)}$.

%%%%%%%%%%%%%%%%%%%%%%%%%%%%%%%
\textbf{Tools for $\ffg$s.}
Solving $\ffg$s requires different methods to solving normal form games \citep{shoham:08} due to their continuous nature. We therefore develop the following basic tools.

First, the \textbf{curry} operator converts a two-player game into a \emph{function from agents to objectives}
\begin{align}
	\label{eq:curry}
	\Big[\phi:W\times W & \longrightarrow \bR\Big] 
	 \xrightarrow{\text{ curry }}
	& \Big[W \longrightarrow \big[ W \longrightarrow \bR\big]\Big]\quad\\
	\phi(\vt,\wt) & & \wt\mapsto \phi_\wt(\bullet):= \phi(\bullet, \wt)
\end{align}
Second, an \textbf{approximate best-response oracle} that, given agent $\vt$ and objective $\phi_\wt(\bullet)$, returns a new agent $\vt' := \oracle(\vt,\phi_\wt(\bullet))$ with $\phi_\wt(\vt') > \phi_\wt(\vt)+\epsilon$, if possible. The oracle could use gradients, reinforcement learning or evolutionary algorithms.

Third, given a population $\pop$ of $n$ agents, the $(n\times n)$ antisymmetric \textbf{evaluation matrix} is
\begin{equation}
    \bA_\pop := \Big\{\phi(\wt_i,\wt_j)
    \;:\; (\wt_i,\wt_j)\in \pop\times\pop\Big\}=: \phi(\pop\otimes\pop).
\end{equation}
Fourth, we will use the (not necessarily unique) \textbf{Nash equilibrium} on the zero-sum matrix game specified by $\bA_\pop$. 

Finally, we use the following \textbf{game decomposition.} Suppose $W$ is a compact set equipped with a probability measure. The set of integrable antisymmetric functions on $W$ then forms a vector space. Appendix~\ref{s:hodge} shows the following:

\begin{thm}[game decomposition]\label{thm:hodge}
	Every $\ffg$ decomposes into a sum of a transitive and cyclic game
	\begin{equation}
        \ffg = \text{transitive game} \oplus \text{cyclic game}.
    \end{equation}
    with respect to a suitably defined inner product. 
\end{thm}
Transitive and cyclic games are discussed below. Few games are purely transitive or cyclic. Nevertheless, understanding these cases is important since general algorithms should, at the very least, work in both special cases.

%%%%%%%%%%%%%%%%%%%%%%%%%%%%%%%
\subsection{Transitive games} 
A game is \textbf{transitive} if there is a `rating function' $f$ such that performance on the game is the difference in ratings:
\begin{equation}
    \label{eq:transitive}
    \phi(\vt,\wt) = f(\vt) - f(\wt).
\end{equation}
In other words, if  $\phi$ admits a `subtractive factorization'. 

%%%%%%%%%%%%%%%%%%%%%%%%%%%%%%%
\textbf{Optimization (training against a fixed opponent).}
Solving a transitive game reduces to finding%\woj{, for some fixed $\wt$,}
\begin{equation}
    \vt^* := \argmax_{\vt\in W} \phi_\wt(\vt)
    = \argmax_{\vt\in W} f(\vt).
\end{equation}

Crucially, the choice of opponent $\wt$ makes no difference to the solution. The simplest learning algorithm is thus to train against a fixed opponent, see algorithm~\ref{alg:fixed}. 

\begin{algorithm}
    \caption{Optimization (against a fixed opponent)}\label{alg:fixed}
\begin{algorithmic}
    \STATE {\bfseries input:} opponent $\wt$; agent $\vt_1$
    \STATE fix objective $\phi_\wt(\bullet)$
    \FOR{$t=1,\ldots, T$}
        \STATE $\vt_{t+1} 
        \leftarrow \oracle\big(\vt_{t}, \phi_\wt(\bullet)\big)$ 
    \ENDFOR
    \STATE {\bfseries output:} $\vt_{T+1}$
\end{algorithmic}
\end{algorithm}

%%%%%%%%%%%%%%%%%%%%%%%%%%%%%%%
\textbf{Monotonic games} generalize transitive games. An $\ffg$ is monotonic if there is a monotonic function $\sigma$ such that 
\begin{equation}
    \label{eq:monotonic}
    \phi(\vt,\wt) = \sigma\big( f(\vt) - f(\wt)\big).
\end{equation}
For example, \citet{Elo78} models the probability of one agent beating another by
\begin{equation}
    \label{eq:elo}
    P(\vt\succ \wt) = \sigma\big( f(\vt) - f(\wt)\big)
    \text{ for }
    \sigma(x) = \frac{1}{1 + e^{-\alpha \cdot x}}
\end{equation}
for some $\alpha>0$, where $f$ assigns Elo ratings to agents. The model is widely used in Chess, Go and other games.

Optimizing against a fixed opponent fares badly in monotonic games. Concretely, if Elo's model holds then training against a much weaker opponent yields no learning signal because the gradient vanishes $\grad_\vt\phi(\vt_t,\wt) \approx 0$ once the sigmoid saturates when $f(\vt_t)\gg f(\wt)$. 

\textbf{Self-play (algorithm~\ref{alg:selfplay})} 
generates a sequence of opponents. Training against a sequence of opponents of increasing strength prevents gradients from vanishing due to large skill differentials, so self-play is well-suited to games modeled by eq.~\eqref{eq:monotonic}. Self-play has proven effective in Chess, Go and other games \citep{silver:18, shedivat:18}.

\begin{algorithm}
    \caption{Self-play}\label{alg:selfplay}
\begin{algorithmic}
    \STATE {\bfseries input:} agent $\vt_1$
    \FOR{$t=1,\ldots, T$}
        \STATE $\vt_{t+1} 
        \leftarrow \oracle\big(\vt_{t}, \phi_{\vt_t}(\bullet)\big)$ 
    \ENDFOR
    \STATE {\bfseries output:} $\vt_{T+1}$
\end{algorithmic}
\end{algorithm}

Self-play \emph{is} an open-ended learning algorithm: it poses and masters a sequence of objectives, rather than optimizing a pre-specified objective. However, self-play assumes transitivity: that local improvements ($\vt_{t+1}$ beats $\vt_t$) imply global improvements ($\vt_{t+1}$ beats $\vt_1, \vt_2, \ldots, \vt_t$). The assumption fails in nontransitive games, such as the disc game below. Since performance is nontransitive, improving against one agent does not guarantee improvements against others. 

%%%%%%%%%%%%%%%%%%%%%%%%%%%%%%%
\subsection{Cyclic games} 

A game is \textbf{cyclic} if 
\begin{equation}
    \label{eq:cyclic}
    \int_W \phi(\vt,\wt)\cdot d\wt = 0
    \quad\text{for all}\quad
    \vt\in W.
\end{equation}
In other words, wins against some agents are necessarily counterbalanced with losses against others. Strategic cycles often arise when agents play simultaneous move or imperfect information games such as rock-paper-scissors, poker, or StarCraft.

\begin{eg}[Disc game]\label{eg:disc}
    Fix $k>0$. Agents are $W = \{\x\in\bR^2\,:\,\|\x\|_2^2\leq k \}$ with the uniform distribution. Set 
    \begin{equation}
        \phi(\vt, \wt) 
        = \vt^\intercal \cdot\left(\begin{matrix}
            0 & -1 \\
            1 & 0
        \end{matrix}\right)\cdot \wt
        = v_1w_2 - v_2w_1.
    \end{equation}
    The game is cyclic, see figure~\ref{f:rps}A.       
\end{eg}

\begin{eg}[Rock-paper-scissors embeds in disc game]\label{eg:rps} 
    Set
    $\rock_\epsilon = \frac{\sqrt{3}\epsilon}{2} (\cos 0, \sin0)$, $\paper_\epsilon = \frac{\sqrt{3}\epsilon}{2} (\cos \frac{2\pi}{3},\sin \frac{2\pi}{3})$ and $\scissors_\epsilon = \frac{\sqrt{3}\epsilon}{2} (\cos \frac{4\pi}{3},\sin \frac{4\pi}{3})$ to obtain
	\begin{equation}
        \bA_{\{\rock_\epsilon,\paper_\epsilon,\scissors_\epsilon\}} =  \left[\begin{matrix}
            0 & \epsilon^2 & -\epsilon^2 \\
            -\epsilon^2 & 0 & \epsilon^2 \\
            \epsilon^2 & -\epsilon^2 & 0
        \end{matrix}\right].
    \end{equation}
    Varying $\epsilon\in[0,1]$ yields a family of $\rock$-$\paper$-$\scissors$ interactions that trend deterministic as $\epsilon$ increases, see figure~\ref{f:rps}B.
\end{eg}

\begin{figure*}[t]
    {\center
    \includegraphics[width=.95\textwidth]{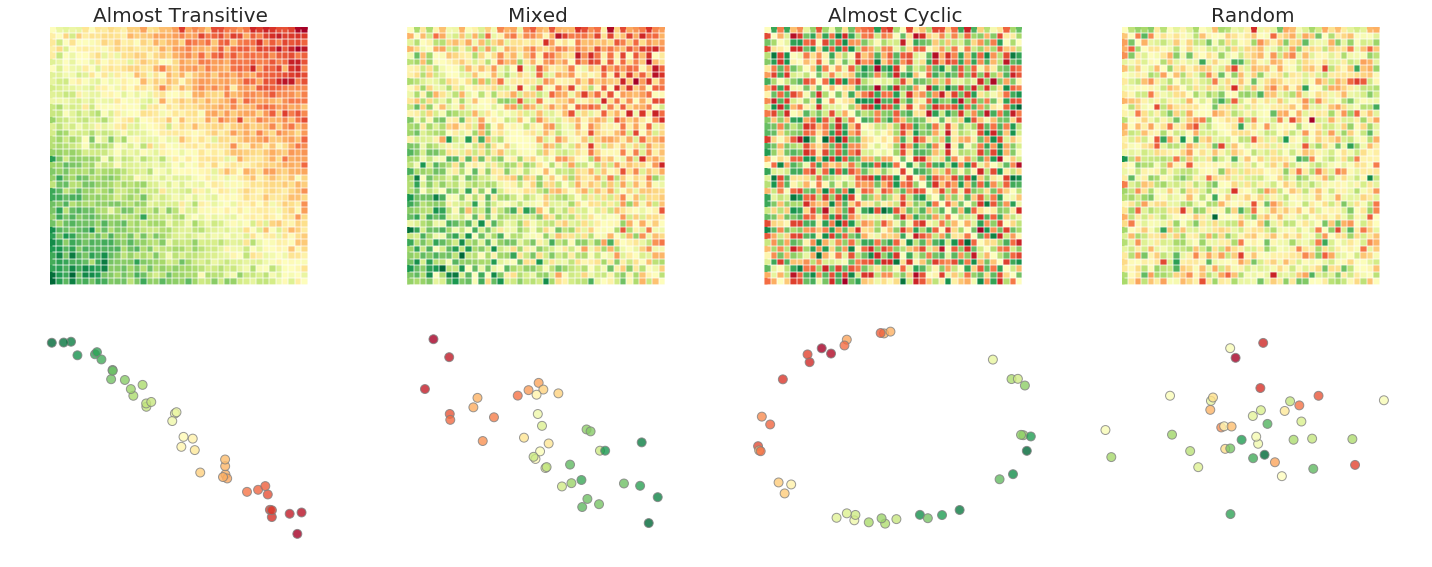}}
    \vspace{-5mm}
    \caption{\emph{Low-dim gamescapes of various basic game structures.}
    \textbf{Top row:}
    Evaluation matrices of populations of 40 agents each; colors vary from red to green as $\phi$ ranges over $[-1,1]$. 
    \textbf{Bottom row:}
    2-dim embedding obtained by using first 2 dimensions of Schur decomposition of the payoff matrix; Color corresponds to average payoff of an agent against entire population;
    $\egs$ of the transitive game is a line;
    $\egs$ of the cyclic game is two-dim near-circular polytope given by convex hull of points.
    For extended version see Figure~\ref{f:embeddings} in the Appendix.
    }
    \label{f:two_leaderboards}
\end{figure*}

%%%%%%%%%%%%%%%%%%%%%%%%%%
Our goal is to extend self-play to general $\ffg$s.
The success of optimization and self-play derives from \textbf{(i)} repeatedly applying a local operation that  \textbf{(ii)} improves a transitive measure. If the measure is \emph{not} transitive, then applying a sequence of local improvements can result in no improvement at all. Our goal is thus to find practical substitutes for \textbf{(i)} and \textbf{(ii)} in general $\ffg$s.

%!TEX root = main.tex
%%%%%%%%%%%%%%%%%%%%%%%%%%%%%%%
\section{Functional and Empirical Gamescapes}
\label{s:gamescapes}

Rather than trying to find a single dominant agent which may not exist, we seek to find all the atomic components in ``strategy space'' of a zero-sum game. That is, we aim to discover the underlying strategic dimensions of the game, and the best ways of executing them. Given such knowledge, when faced with a new opponent, we will not only be able to react to its behavior conservatively (using the Nash mixture to guarantee a tie), but will also be able to optimally exploit the opponent. As opposed to typical game-theoretic solutions, we do not seek a single agent or mixture, but rather a population that embodies a complete understanding of the strategic dimensions of the game.

To formalize these ideas we introduce \textbf{gamescapes}, which geometrically represent agents in functional form games. We show some general properties of these objects to build intuitions for the reader. Finally we introduce two critical concepts: \textbf{population performance}, which measures the progress in performance of populations, and \textbf{effective diversity}, which quantifies the coverage of the gamescape spanned by a population. Equipped with these tools we present algorithms that guarantee iterative improvements in $\ffg$s.

\begin{defn}
	The \textbf{functional gamescape} ($\fgs$) of $\phi:W\times W\rightarrow \bR$ is the convex set
	\begin{equation}
		\cG_\phi := \hull \left(\Big\{\phi_\wt(\bullet)\,:\,\wt\in W\Big\}\right)\subset \cC(W,\bR),
	\end{equation}
	where $\cC(W,\bR)$ is the space of real-valued functions on $W$.
	
	Given population $\pop$ of $n$ agents with evaluation matrix $\bA_\pop$, the corresponding \textbf{empirical gamescape} ($\egs$) is
    \begin{equation}
	    \cG_\pop := \big\{\text{convex mixtures of rows of }\bA_\pop\big\}.
    \end{equation}
\end{defn}
The $\fgs$ represents all the mixtures of objectives implicit in the game. We cannot work with the $\fgs$ directly because we cannot compute $\phi_\wt(\bullet)$ for infinitely many agents. The $\egs$ is a tractable proxy \citep{Wellman06}. The two gamescapes represent all the ways agents can-in-principle and are-actually-observed-to interact respectively. The remainder of this section collects basic facts about gamescapes.

%%%%%%%%%%%%%%%%%%%%%%%%%%%%%%%
\textbf{Optimization landscapes}
are a special case of gamescapes. If $\phi(\vt, \wt) =f(\vt)-f(\wt)$ then the $\fgs$ is, modulo constant terms, a single function $\cG_\phi=\big\{f(\bullet) - f(\wt) : \wt\in W\big\}$. The $\fgs$ degenerates into a landscape where, for each agent $\vt$ there is a unique direction $\grad \phi_\wt(\vt) = \grad f(\vt)$ in weight space which gives the steepest performance increase \emph{against all opponents}. In a monotonic game, the gradient is $\grad \phi_\wt(\vt) = \sigma'\cdot \grad f(\vt)$. There is again a single steepest direction $\grad f(\vt)$, with tendency to vanish controled by the ratings differential $\sigma' = \sigma'\big(f(\vt) - f(\wt)\big) \geq 0$. 
%%%%%%%%%%%%%%%%%%%%%%%%%%%%%%%

%%%%%%%%%%%%%%%%%%%%%%%%%%%%%%%
\textbf{Redundancy.}
First, we argue that gamescapes are more fundamental than evaluation matrices. Consider 
\begin{equation}
    \left[\begin{matrix}
        0 & 1 & -1 \\
        -1 & 0 & 1 \\
        1 & -1 & 0
    \end{matrix}\right]
    \quad\text{and}\quad
    \left[\begin{matrix}
        0 & 1 & -1 & -1 \\
        -1 & 0 & 1 & 1 \\
        1 & -1 & 0 & 0 \\
        1 & -1 & 0 & 0 
    \end{matrix}\right].
\end{equation}
The first matrix encodes rock-paper-scissors interactions; the second is the same, but with two copies of scissors. The matrices are difficult to compare since their dimensions are different. Nevertheless, the gamescapes are equivalent triangles embedded in $\bR^3$ and $\bR^4$ respectively. 

\begin{prop}\label{prop:invariance}
    An agent in a population is redundant if it is behaviorally identical to a convex mixture of other agents. The $\egs$ is \textbf{invariant} to redundant agents. 
\end{prop}

Invariance is explained in appendix~\ref{s:proofs}.

%%%%%%%%%%%%%%%%%%%%%%%%%%%%%%%
\textbf{Dimensionality.}
The dimension of the gamescape is an indicator of the complexity of both the game and the agents playing it. In practice we find many $\ffg$s have a low dimensional latent structure.

Figure~\ref{f:two_leaderboards} depicts evaluation  matrices of four populations of 40 agents. Although the gamescapes could be 40-dim, they turn out to have one- and two-dim representations. The dimension of the $\egs$ is determined by the rank of the evaluation matrix. 

\begin{prop}\label{prop:schur}
    The $\egs$ of $n$ agents in population $\pop$ can be represented in $\bR^r$, where $r=\rank(\bA_\pop)\leq n$.
\end{prop}

A low-dim representation of the $\egs$ can be constructed via the Schur decomposition, which is the analog of PCA for antisymmetric matrices \citep{reval:18}. The length of the longest strategic cycle in a game gives a lower-bound on the dimension of its gamescape:

\begin{eg}[latent dimension of long cycles]\label{eg:long}
    Suppose $n$ agents form a long cycle: $\pop = \{\vt_1\xrightarrow{\text{beats}}\vt_2\rightarrow\cdots\rightarrow\vt_n\xrightarrow{\text{beats}}\vt_1\}$. Then $\text{rank}(\bA_\pop)$ is $n-2$ if $n$ is even and $n-1$ if $n$ is odd.
\end{eg}

%%%%%%%%%%%%%%%%%%%%%%%%%%%%%%%
\textbf{Nash equilibria}
in a symmetric zero-sum game are (mixtures of) agents that beat or tie all other agents. Loosely speaking, they replace the notion of best agent in games where there is no best agent. Functional Nash equilibria, in the $\fgs$, are computationally intractable so we work with empirical Nash equilibria over the evaluation matrix. 

\begin{prop}\label{prop:nash}
    Given population $\pop$, the empirical Nash equilibria are 
    \begin{equation}
	    \cN_\pop 
	    = \{\bp \text{ distribution}: \bp^\intercal\bA_\pop\succeq\bZ\}.
    \end{equation}
\end{prop}
In other words, Nash equilibria correspond to points in the  empirical gamescape that intersect the positive quadrant $\{\x \in \cG_\pop: \x\succeq\bZ\}$. The positive quadrant thus provides a set of directions in weight space to aim for when training new agents, see $\psro$ below. 

%%%%%%%%%%%%%%%%%%%%%%%%%%%%%%%
\textbf{The gap between the $\egs$ and $\fgs$.}
Observing $\rock$-$\paper$ interactions yields different conclusions from observing $\rock$-$\paper$-$\scissors$ interactions;
it is always possible that an agent that appears to be dominated is actually part of a partially observed cycle. Without further assumptions about the structure of $\phi$, it is impossible to draw strong conclusions about the nature of the $\fgs$ from the $\egs$ computed from a finite population. The gap is analogous to the exploration problem in reinforcement learning. To discover unobserved dimensions of the $\fgs$ one could train against randomized distributions over opponents, which would eventually find them all.

%%%%%%%%%%%%%%%%%%%%%%%%%%%%%%%
\subsection{Population performance}

If $\phi(\vt,\wt)=f(\vt)-f(\wt)$ then improving performance of agent $\vt$ reduces to increasing $f(\vt)$. In a cyclic game, the performance of individual agents is meaningless: beating one agent entails losing against another by eq.~\eqref{eq:cyclic}. We therefore propose a \emph{population} performance measure. 

\begin{defn}\label{def:perf}
	Given populations $\pop$ and $\popq$, let $(\bp,\bq)$ be a Nash equilibrium of the zero-sum game on $\bA_{\pop, \popq}:=\phi(\vt,\wt)_{\vt\in\pop, \wt\in\popq}$. The \textbf{relative population performance} is 
    \begin{equation}
        \perf(\pop, \popq) 
        := \bp^\intercal\cdot \bA_{\pop, \popq} \cdot \bq 
        = \sum_{i,j=1}^{n_1,n_2} A_{ij} \cdot p_i q_j.
    \end{equation}    
\end{defn}
\begin{prop}\label{prop:perf}
	\textbf{(i)} 
	Performance $\perf$ is independent of the choice of Nash equilibrium. 
	\textbf{(ii)} 
	If $\phi$ is monotonic then performance compares the best agents in each population
	\begin{equation}
    	\perf(\pop,\popq) 
    	= \max_{\vt\in\pop} f(\vt) 
    	- \max_{\wt\in\popq} f(\wt).
	\end{equation}
	\textbf{(iii)}
	If $\hull(\pop) \subset \hull(\popq)$ then $\perf(\pop,\popq)\leq0$ and $\perf(\pop, \popr)\leq \perf(\popq,\popr)$ for \textbf{any} population $\popr$.
\end{prop}
The first two properties are sanity checks. Property \emph{\textbf{(iii)}} implies growing the polytope spanned by a population improves its performance \emph{against any other population}.

Consider the concentric rock-paper-scissors populations in figure~\ref{f:rps}B and example~\ref{eg:rps}. The Nash equilibrium is $(0,0)$, which is a uniform mixture over any of the populations. Thus, the relative performance of any two populations is zero. However, the outer population is better than the inner population at \emph{exploiting} an opponent that only plays, say, rock because the outer version of paper wins more deterministically than the inner version.

Finding a population that contains the Nash equilibrium is necessary but not sufficient to fully solve an $\ffg$. For example, adding the ability to always force a tie to an $\ffg$ makes finding the Nash trivial. However, the game can still exhibit rich strategies and counter-strategies that are worth discovering.

%%%%%%%%%%%%%%%%%%%%%%%%%%%%%%%
\subsection{Effective diversity}
Measures of diversity typically quantify differences in weights or behavior of agents but ignore performance. \emph{Effective} diversity measures the variety of effective agents (agents with support under Nash):

\begin{defn}\label{def:diverse}
    Denote the rectifier by $\floor*{x}_+:= x$ if $x\geq 0$ and $\floor*{x}_+ := 0$ otherwise. Given population $\pop$, let $\bp$ be a Nash equilibrium on $\bA_\pop$. The \textbf{effective diversity} of the population is:
    \begin{equation}
        \diverse(\pop) 
        := \bp^\intercal\cdot  \floor*{\bA_\pop}_+ \cdot \bp
        = \sum_{i,j=1}^n \floor*{\phi(\wt_i,\wt_j)}_+ \cdot p_ip_j.
    \end{equation}
\end{defn}
Diversity quantifies how the best agents (those with support in the maximum entropy Nash) exploit each other. If there is a dominant agent then diversity is zero. 

Effective diversity is a matrix norm, see appendix~\ref{s:div_norm}. It measures the $\ell_{1,1}$ volume spanned by Nash supported agents. In figure~\ref{f:rps}B, there are four populations spanning concentric gamescapes: the Nash at $(0,0)$ and three variants of $\rock$-$\paper$-$\scissors$. Going outwards to large gamescapes yields agents that are more diverse and better exploiters.

\begin{figure}[t]  
    {\center
    \includegraphics[width=.5\textwidth]{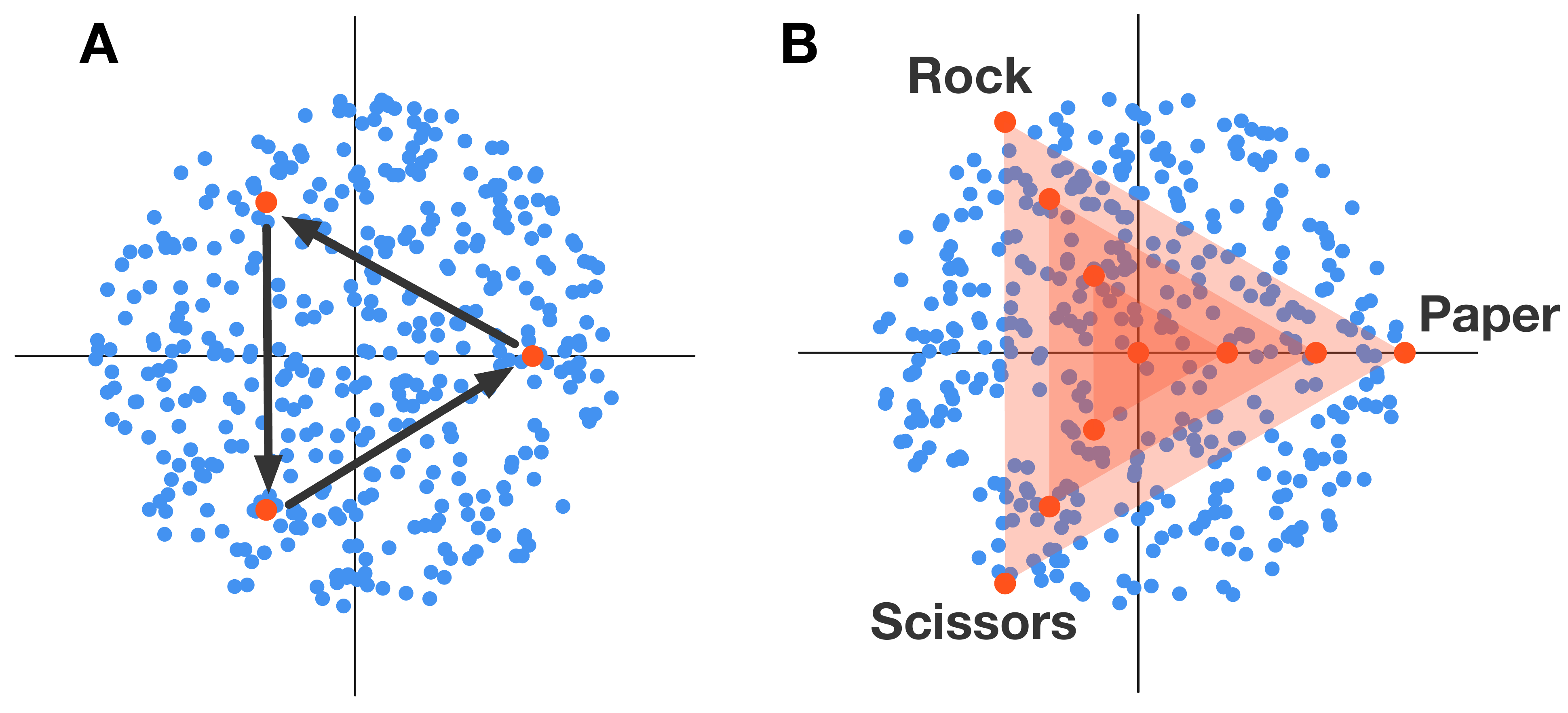}}
    \vspace{-5mm}
    \caption{\emph{The disc game.} \textbf{A:} A set of possible agents from the disc game is shown as blue dots. Three agents with non-transitive rock-paper-scissors relations are visualized in red. \textbf{B:} Three concentric gamescapes spanned by populations with rock-paper-scissor interactions of increasing strength.}
    \label{f:rps}
\end{figure}

%!TEX root = main.tex
%%%%%%%%%%%%%%%%%%%%%%%%%%%%%%%%%%%%%%%%%%%%%%%%%%
\section{Algorithms}
\label{l_sect_alg}

We now turn attention to constructing objectives that when trained against, produce new, effective agents. We present two algorithms that construct a sequence of fruitful local objectives that, when solved, iteratively add up to transitive population-level progress. Importantly, these algorithms output populations, unlike self-play which outputs single agents.

Concretely, we present algorithms that expand the empirical gamescape in useful directions. Following \citet{lanctot:17}, we assume access to a subroutine, or \textbf{oracle}, that finds an approximate best response to any mixture $\sum_{i} p_i \phi_{i}(\wt_i)$ of objectives. The subroutine could be a gradient-based, reinforcement learning or evolutionary algorithm. The subroutine returns a vector in weight-space, in which existing agents can be shifted to create new agents. Any mixture constitutes a valid training objective. However, many mixtures do not grow the gamescape, because the vector could point towards redundant or weak agents.

%%%%%%%%%%%%%%%%%%%%%%%%%%%%%%%
\subsection{Response to the Nash ($\psro$)}
\label{l_sect_psro_nash}

Since the notion of `the best agent' -- one agent that beast all others -- does not necessarily exist in nontransitive games, a natural substitute is the mixture over the Nash equilibrium on the most recent population $\pop_t$. The policy space response to the Nash ($\psro$) iteratively generates new agents that are approximate best responses to the Nash mixture. If the game is transitive then $\psro$ degenerates to self-play. The algorithm is an extension of the double oracle algorithm \citep{mcmahan:03} to $\ffg$s, see also \citep{zinkevich:07, hansen:08}.

\begin{algorithm}
    \caption{Response to Nash ($\psro$)}\label{alg:psro}
\begin{algorithmic}
   \STATE {\bfseries input:} population $\pop_1$ of agents
    \FOR{$t=1, \ldots, T$}
    \STATE $\bp_t \leftarrow$ Nash on $\bA_{\pop_t}$
    \STATE $\vt_{t+1} \leftarrow$ $\oracle\big(\vt_t, \sum_{\wt_i\in\pop_t} \bp_t[i]\cdot \phi_{\wt_i}(\bullet)\big)$
    \STATE $\pop_{t+1} \leftarrow  \pop_t\cup\{\vt_{t+1}\}$
    \ENDFOR
    \STATE {\bfseries output:} $\pop_{T+1}$
\end{algorithmic}
\end{algorithm}

The following result shows that $\psro$ strictly enlarges the empirical gamescape:

\begin{prop}\label{prop:oracle_guarantee}
    If $\bp$ is a Nash equilibrium on $\bA_\pop$ and $\sum_i p_i\phi_{\wt_i}(\vt)>0$, then adding $\vt$ to $\pop$ strictly enlarges the empirical gamescape: $\cG_\pop \subsetneq \cG_{\pop\cup\{\vt\}}$.
\end{prop}

A failure mode of $\psro$ arises when the Nash equilibrium of the game is contained in the empirical gamescape. For example, in the disc game in figure~\ref{f:rps} the Nash equilibrium of the entire $\ffg$ is the agent at the origin $\wt=(0,0)$. If a population's gamescape contains $\wt=(0,0)$ -- which is true of any $\rock$-$\paper$-$\scissors$ subpopulation -- then $\psro$ will not expand the gamescape because there is no $\epsilon$-better response to $\wt=(0,0)$. The next section presents an algorithm that uses \emph{niching} to meaningfully grow the gamescape, even after finding the Nash equilibrium of the $\ffg$.

%%%%%%%%%%%%%%%%%%%%%%%%%%%%%%%
\textbf{Response to the uniform distribution ($\uresp$).}
A closely related algorithm is fictitious (self-)play \citep{brown:51, leslie:06, heinrich:15}. The algorithm finds an approximate best-response to the uniform distribution on agents in the current population: $\sum_{\wt_i\in\pop_t} \phi_{\wt_i}(\bullet)$. $\uresp$ has guarantees in matrix form games and performs well empirically (see below). However, we do not currently understand its effect on the gamescape.

%%%%%%%%%%%%%%%%%%%%%%%%%%%%%%%
\subsection{Response to the rectified Nash ($\tfold$)}
\label{l_sect_psro_rect_nash}

Response to the \emph{rectified} Nash ($\tfold$), introduces game-theoretic niches. Each effective agent -- that is each agent with support under the Nash equilibrium -- is trained against the Nash-weighted mixture of agents that it beats or ties. Intuitively, the idea is to encourage agents to `amplify their strengths and ignore their weaknesses'. 

A special case of $\tfold$ arises when there is a dominant agent in the population, that beats all other agents. The Nash equilibrium is then concentrated on the dominant agent, and $\tfold$ degenerates to training against the best agent in the population, which can be thought of as a form of self-play (assuming the best agent is the most recent).

\begin{algorithm}
    \caption{Response to rectified Nash ($\tfold$)}\label{alg:threefold}
\begin{algorithmic}
   \STATE {\bfseries input:} population $\pop_1$
   \FOR{$t=1,\ldots, T$}
        \STATE $\bp_t \leftarrow$ Nash on $\bA_{\pop_t}$
        \FOR{agent $\vt_t$ with positive mass in $\bp_t$}
            \STATE $\vt_{t+1} \leftarrow  \oracle\big(\vt_t, \sum_{\wt_i\in\pop_t} \bp_t[i]\cdot \floor*{\phi_{\wt_i}(\bullet)}_+
            \big)$ 
        \ENDFOR
        \STATE $\pop_{t+1}\leftarrow\pop_t\cup\{\vt_{t+1}:$ updated above$\}$
    \ENDFOR
    \STATE {\bfseries output:} $\pop_{T+1}$
\end{algorithmic}
\end{algorithm}

\begin{figure}[t]  
    {\center
    \includegraphics[width=.5\textwidth]{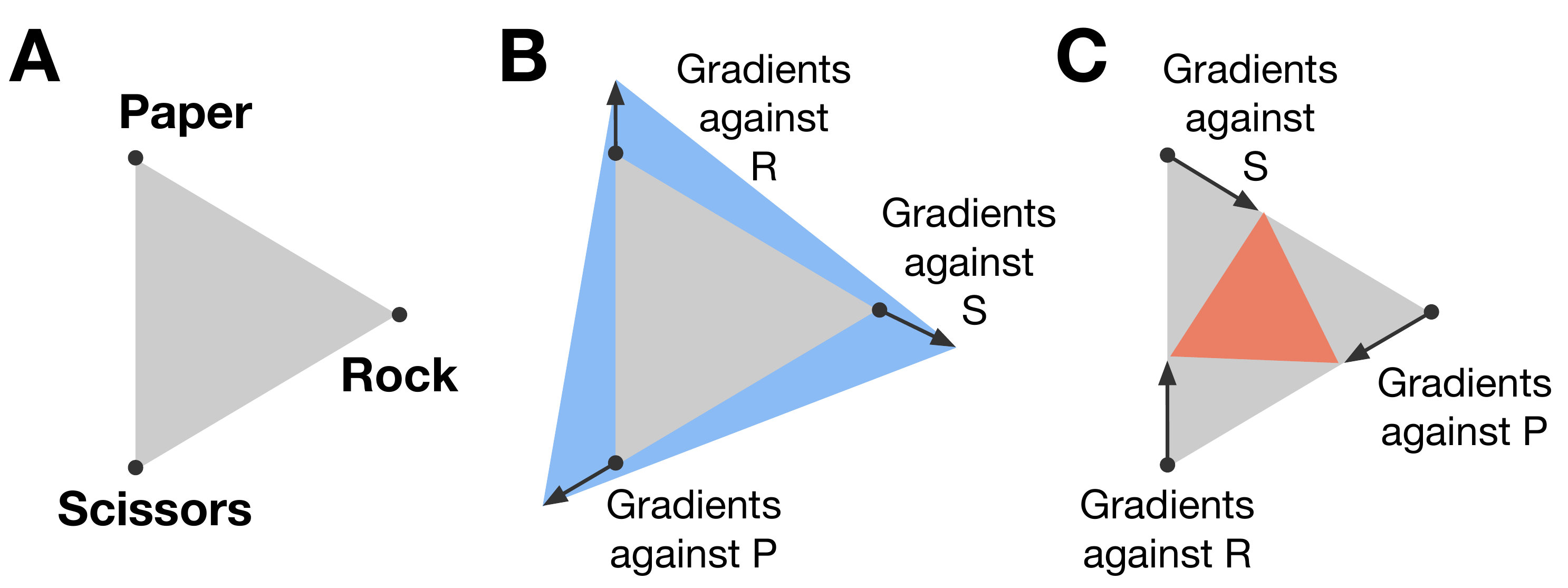}}
    \vspace{-5mm}
    \caption{
    \textbf{A:}
    Rock-paper-scissors.
    \textbf{B:} Gradient updates obtained from $\tfold$, amplifying strengths, grow gamescape (gray to blue).
    \textbf{C:} Gradients obtained by optimizing agents to reduces their losses shrink gamescape (gray to red).
    }
    \label{f:threefold}
\end{figure}

\begin{prop}\label{prop:threefold}
    The objective constructed by rectified Nash response is effective diversity, definition~\ref{def:diverse}.
\end{prop}
Thus, $\tfold$ amplifies the positive coordinates, of the Nash-supported agents, in their rows of the evaluation matrix. A pathological mode of $\tfold$ is when there are many extremely local niches. That is, every agent admits a specific exploit that does not generalize to other agents. $\tfold$ will grow the gamescape by finding these exploits, generating a large population of highly specialized agents.

\textbf{Rectified Nash responses in the disc game} (example~\ref{eg:disc})\textbf{.}
The disc game embeds many subpopulations with rock-paper-scissor dynamics. As the polytope they span expands outwards, the interactions go from noisy to deterministic. The disc game is differentiable, so we can use gradient-based learning for the oracle in $\tfold$.  Figure~\ref{f:threefold}B depicts the gradients resulting from training each of rock, paper and scissors against the agent it beats. Since the gradients point \emph{outside} the polytope, training against the rectified Nash mixtures expands the gamescape.

%%%%%%%%%%%%%%%%%%%%%%%%%%%%%%%%%%%%%%%%%%%%
\textbf{Why ignore weaknesses?}
A natural modification of $\tfold$ is to train effective agents against effective agents that they \emph{lose} to. In other words, to force agents to improve their weak points whilst taking their strengths for granted. Figure~\ref{f:threefold}C shows the gradients that would be applied to each of rock, paper and scissors under this algorithm. They point \emph{inwards}, contracting the gamescape. Training rock against paper would make it more like scissors; similarly training paper against scissors would make it more like rock and so on. Perhaps counter-intuitively, building objectives using the weak points of agents does not encourage diverse niches.

%!TEX root = main.tex
%%%%%%%%%%%%%%%%%%%%%%%%%%%%%%%
\section{Experiments}
\label{s:exp}

We investigated the performance of the proposed algorithms in two highly nontransitive resource allocation games.

\begin{figure}[t]  
    {\center
    \includegraphics[width=.5\textwidth]{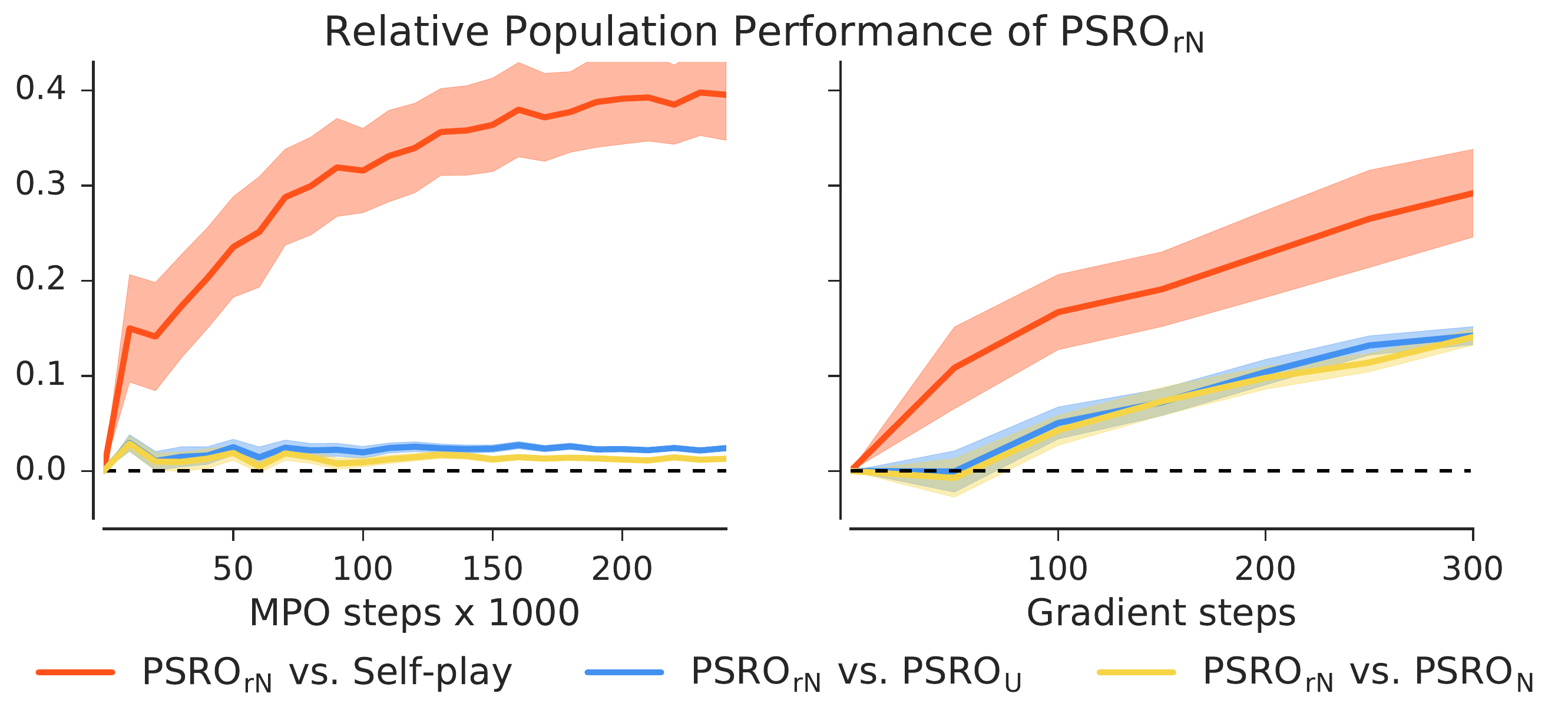}}
    \vspace{-5mm}
    \caption{Performance of $\tfold$ relative to self-play, $\uresp$ and $\psro$ on Blotto (left) and Differentiable Lotto (right). In all cases, the relative performance of $\tfold$ is positive, and therefore outperforms the other algorithms.}
    \label{f:blotto}
\end{figure}

\begin{figure*}[h!]
    \includegraphics[width=0.70\textwidth]{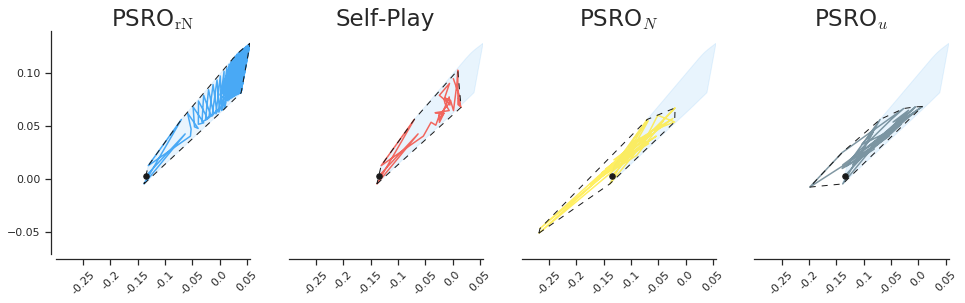}
    \includegraphics[width=0.29\textwidth]{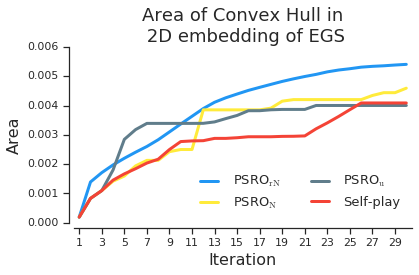}
    \caption{Visualizations of training progress in Differentiable Lotto experiment. \textbf{Left:} Comparison of trajectories taken by each algorithm in the 2-dim Schur embedding of the $\egs$; a black dot represents first agent found by the algorithm and a dashed line represents the convex full. Shaded blue region shows area of the convex hull of $\tfold$. Notice the $\tfold$ consistent expansion of the convex hull through ladder-like movements. See Figure~\ref{f:schurs} for an extended version. \textbf{Right:} Area of convex hull spanned by populations over time. Note that only $\tfold$ consistently increases the convex hull in all iterations.}
    \label{f:gs_through_time}
\end{figure*}

%%%%%%%%%%%%%%%%%%%%%%%%%%%%%%%
\textbf{Colonel Blotto} is a resource allocation game that is often used as a model for electoral competition. Each of two players has a budget of $c$ coins which they simultaneously distribute over a fixed number of areas.
Area $a_i$ is won by the player with more coins on $a_i$. The player that wins the most areas wins the game. Since Blotto is not differentiable we use maximum a posteriory policy optimization (MPO)
\citep{abdolmaleki:18} as best response oracle.
MPO is an inference-based policy optimization algorithm; many other reinforcement learning algorithms could be used. 

%%%%%%%%%%%%%%%%%%%%%%%%%%%%%%%
\textbf{Differentiable Lotto} is inspired by continuous Lotto \citep{hart:08}. The game is defined over a fixed set $\cC$ of $c$ `customers', each being a point in $\bR^2$. An agent $(\bp,\vt) = \big\{(p_1,\vt_1),\ldots, (p_k,\vt_k)\big\}$ distributes one unit of mass over $k$ servers, where each server is a point $\vt_i\in\bR^2$. 
    % YB added this for a little intuition here.
Roughly, given two agents  $(\bp,\vt)$ and $(\bq,\wt)$, customers are softly assigned to the nearest servers, determining the agents' payoffs. More formally, 
the payoff is
\begin{equation}
    \phi\big((\bp,\vt),(\bq,\wt)\big) := \sum_{i,j=1}^{c,k} \big(p_jv_{ij} - q_jw_{ij}\big),
\end{equation}
where the scalars $v_{ij}$ and $w_{ij}$ depend on the distance between customer $i$ and the servers:
\begin{equation}
    (v_{i1},\ldots, w_{ik}) := \text{softmax}(-\|\bc_i - \vt_1\|^2,\ldots,-\|\bc_i-\wt_k\|^2).
\end{equation}
The \emph{width} of a cloud of points is the expected distance from the barycenter. We impose agents to have width equal one. We use gradient ascent as our oracle.

\textbf{Experimental setups.}
The experiments examine the performance of self-play, $\uresp$, $\psro$, and $\tfold$. We investigate performance under a fixed computational budget. Specifically, we track queries made to the oracle, following the standard model of computational cost in convex optimization \citep{nemirovski:83}. To compare two algorithms, we report the relative population performance (definition~\ref{def:perf}), of the populations they output. Computing evaluation matrices is expensive, ${\mathcal O}(n^2)$, for large populations. This cost is not included in our computational model since populations are relatively small. The relative cost of evaluations and queries to the oracle depends on the game.

In Blotto, we investigate performance for $a=3$ areas and $c=10$ coins over $k=1000$ games. An agent outputs a vector in $\bR^3$ which is passed to a softmax, $\times10$ and discretized to obtain three integers summing to 10. Differentiable Lotto experiments are from $k=500$ games with $c=9$ customers chosen uniformly at random in the square $[-1, 1]^2$. 

\textbf{Results.}
Fig~\ref{f:blotto} shows the relative population performance, definition~\ref{def:perf}, between $\tfold$ and each of $\psro$, $\uresp$ and self-play: the more positive the number is, the more $\tfold$ outperforms the alternative method. We find that $\tfold$ outperforms the other approaches across a wide range of allowed compute budgets. $\uresp$ and $\psro$ perform comparably, and self-play performs the worst.
Self-play, algorithm~\ref{alg:selfplay}, outputs a single agent, so the above comparison only considers the final agent. If we upgrade self-play to a population algorithm (by tracking all agents produced over), then it still performs the worst in differentiable Lotto, but by a smaller margin. In Blotto, suprisingly, it slightly outperforms $\psro$ and $\uresp$.

Figure~\ref{f:gs_through_time} shows how gamescapes develop during training. From the left panel, we see that $\tfold$ grows the polytope in a more uniform manner than the other algorithms. The right panel shows the area of the empirical gamescapes generated by the algorithms (the areas of the convex hulls). All algorithms increase the area, but $\tfold$ is the only method that increases the area at every iteration.

%!TEX root = main.tex
%%%%%%%%%%%%%%%%%%%%%%%%%%%%%%%
\section{Conclusion}

We have proposed a framework for open-ended learning in two-player symmetric zero-sum games -- where strategies are agents, with a differentiable parametrization. We propose the goal of learning should be \textbf{(i)} to discover the underlying strategic components that constitute the game and \textbf{(ii)} to master each of them. We formalized these ideas using gamescapes, which geometrically represent the latent objectives in games, and provided tools to analyze them. Finally, we proposed and empirically validated a new algorithm, $\tfold$, for uncovering strategic diversity within functional form games.

The algorithms discussed here are simple and generic, providing the foundations for methods that unify modern gradient and reinforcement-based learning with the adaptive objectives derived from game-theoretic considerations. Future work lies in expanding this understanding and applying it to develop practical  algorithms for more complex games.

{

}

%!TEX root = main.tex
%%%%%%%%%%%%%%%%%%%%%%%%%%%%%%%
\setcounter{section}{0}
\vspace{5mm}
\renewcommand{\thesection}{\Alph{section}}
\vspace{5mm}
\noindent
{\textsf{\textbf{\Large{APPENDIX}}}}

%%%%%%%%%%%%%%%%%%%%%%%%%%%%%%%%%%%%%%%%%%%%%%%%%%%%%%%%%%%%%%
\section{Behaviorism}
\label{s:behave}

The paper adopts a strict form of behaviorism: an agent is what it does. More precisely, an agent \emph{is all that it can do}. Agents are operationally characterized by all the ways they can interact with all other agents. The functional gamescape captures all possible interactions; the empirical gamescape captures all observed interactions.

Here, we restrict to the single facet of pairwise interactions reported by the function $\phi$. One could always delve deeper in specific instances. For example, instead of just reporting win/loss probabilities, one could also report summary statistics regarding which units are built or which pieces are taken.

%%%%%%%%%%%%%%%%%%%%%%%%%%%%%%%%%%%%%%%%%%%%%%%%%%%%%%%%%%%%%%
\section{Comparison with GANs and co-evolution}
\label{s:coev}

Generative adversarial networks are two-player zero-sum games played by neural networks \citep{goodfellow:14}. The main difference between GANs and $\ffg$'s is that GANs are not symmetric. The discriminator and generator play qualitatively different roles -- it would be meaningless to plug discriminator weights into the generator or vice versa. In short, the two losses that comprise a GAN have two different semantics. 

There is a substantial literature on co-evolution, which has substantial overlap with our setting \citep{hillis:90, rosin:97, nolfi:98, ficici:00, ficici:01, dejong:01, monroy:06, schmidt:08, popovici:12}. As with GANs, co-evolution often involves \emph{asymmetric} interactions -- such as setter-solver, predators and prey, or organisms and parasites. As a result, the semantics are much richer, but also much more difficult to analyze, and therefore beyond the scope of this paper.

%%%%%%%%%%%%%%%%%%%%%%%%%%%%%%%%%%%%%%%%%%%%%%%%%%%%%%%%%%%%%%

\begin{figure*}[t]
    {\center
    \includegraphics[width=.95\textwidth]{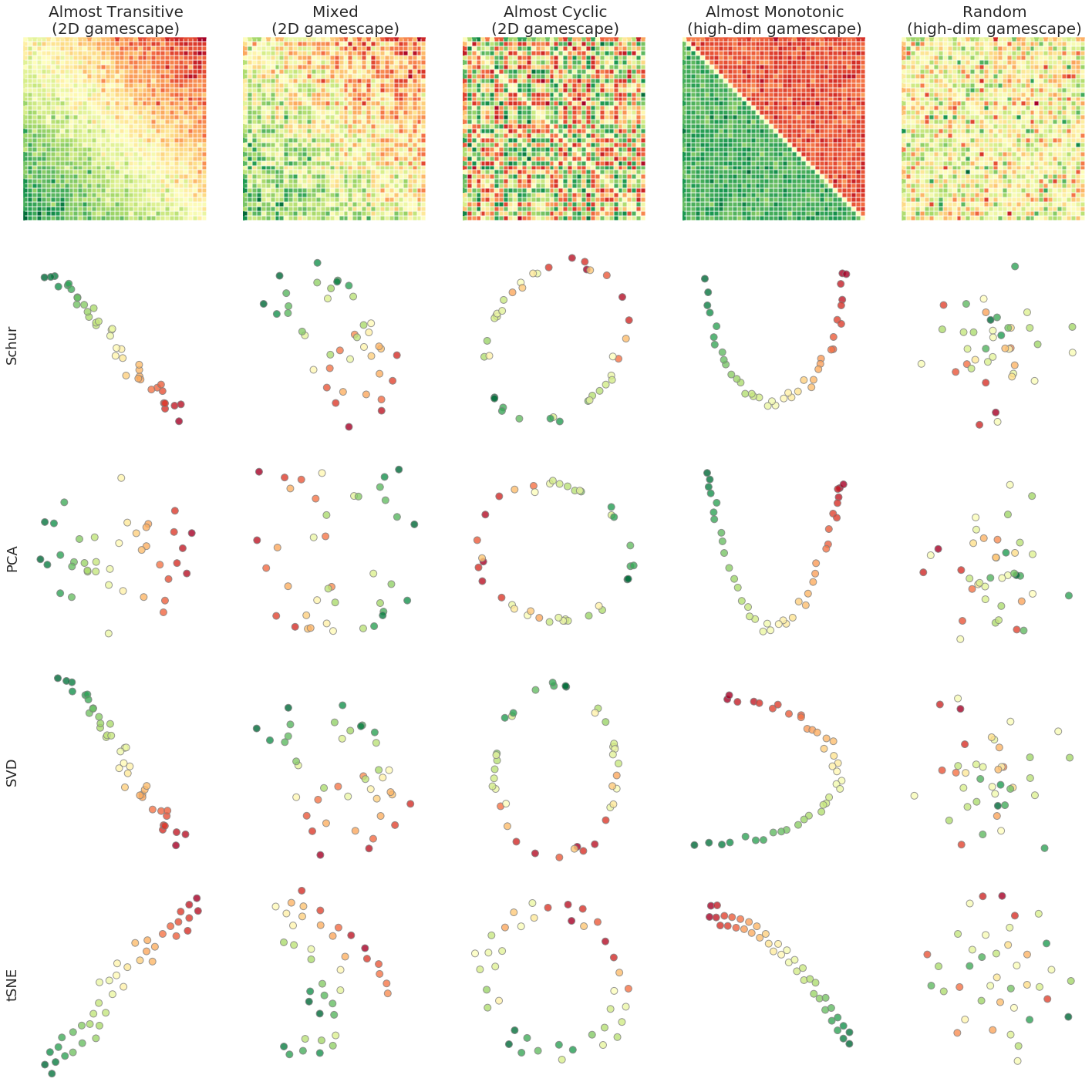}}
    \vspace{-5mm}
    \caption{\emph{Low-dim gamescapes of various basic game structures.}
    Extended version of Figure~\ref{f:two_leaderboards}.
    \textbf{Top row:}
    Evaluation matrices of populations of 40 agents each; colors vary from red to green as $\phi$ ranges over $[-1,1]$. 
    \textbf{Bottom row:}
    2D embedding obtained by using first 2 dimensions of various embedding techniques; Color corresponds to average payoff of an agent against entire population;
    $\egs$ of the transitive game is a line;
    $\egs$ of the cyclic game is two-dim near-circular polytope given by convex hull of points.
    }
    \label{f:embeddings}
\end{figure*}

\begin{figure*}[t]
    {\center
    \includegraphics[width=.85\textwidth]{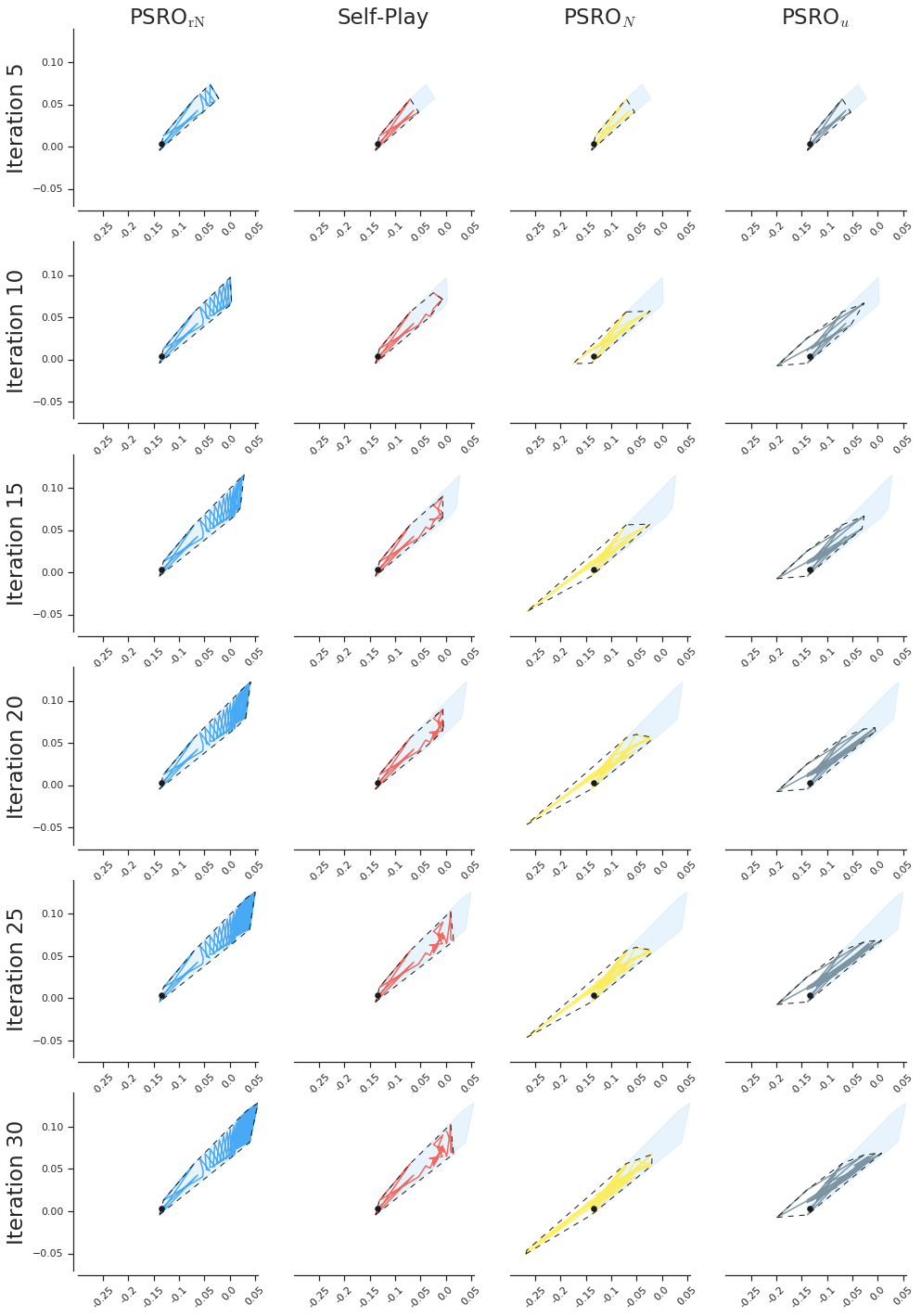}}
    \vspace{-5mm}
    \caption{ Extended version of Figure~\ref{f:gs_through_time}.
    Trajectories in 2D Schur embedding space while training on Differentiable Blotto game. Black dot represents initial agent, and dashed lines correspond to a convex hull of the population. Shaded blue region represents area of the convex hull of the $\tfold$ population at a given iteration.}
    \label{f:schurs}
\end{figure*}

\section{Convex geometry}
\label{s:convex}

A set $S$ contained in some vector space is \textbf{convex} if the lines
\begin{equation}
	L_{\vt,\wt} = \big\{ \alpha\cdot \vt + (1-\alpha)\cdot\wt \,:\, 0\leq\alpha\leq1 \big\}
\end{equation}
are subsets of $S$ for all $\vt,\wt\in S$.

\begin{defn}
    The \textbf{convex hull} of a set $A$ is the intersection of all convex sets containing $A$.
    Alternatively, the convex hull of $A\subset \bR^d$ is 
    \begin{equation}
        \hull(A) := \left\{\sum_{i=1}^n \alpha_i\cdot \wt_i : \balpha\succeq 0, \balpha^\intercal\bO=1, \wt_i\in A\right\}.
    \end{equation}
\end{defn}

The convex hull of a set is necessarily convex. 

\begin{thm}[Krein-Milman]
    An \textbf{extreme point} of a convex set $S$ is a point in $S$ which does not lie in any open line segment joining two points of $S$. If $S$ is convex and compact in a normed space, then $S$ is the closed convex hull of its extreme points.
\end{thm}

\paragraph{Row versus column mixtures.} 
In the main text we define the empirical gamescape $\cG_\bA$ to be all convex mixtures of \emph{rows} of an antisymmetric matrix $\bA$. It follows by antisymmetry of $\bA$ that working with convex mixtures of columns obtains the same polytope up to sign. 

The same holds for functional gamescape, where we chose to work with convex combinations of functions of the form $\phi_\wt(\bullet) = \phi(\bullet, \wt)$, but could have equivalently (up to sign) worked with convex combinations of functions of the form $\phi(\wt, \bullet)$.

%%%%%%%%%%%%%%%%%%%%%%%%%%%%%%%%%%%%%%%%%%%
\section{Proof of theorem~\ref{thm:hodge}}
\label{s:hodge}

We generalize some results from combinatorial Hodge theory \citep{jiang:11, candogan:11, reval:18} to the setting where flows are antisymmetric \emph{functions} instead of antisymmetric \emph{matrices}. 

%%%%%%%%%%%%%%%%%%%%%%%%%%%%%%
\subsection{Combinatorial Hodge theory}

Combinatorial Hodge theory is a combinatorial analog of differential geometry, see \citet{jiang:11, candogan:11, baker:18}. Consider a fully connected graph with vertex set $[n]=\{1,\ldots,n\}$. Assign a \textbf{flow} $\bA_{ij}\in\bR$ to each edge of the graph. The flow in the opposite direction $ji$ is $\bA_{ji} = -\bA_{ij}$, so flows are just $(n\times n)$ antisymmetric matrices. The flow on a graph is analogous to a vector field on a manifold. 

The \textbf{combinatorial gradient} of an $n$-vector $\br$ is the flow: $\ggrad(\br) := \br\bO^\intercal - \bO\br^\intercal$. Flow $\bA$ is a \textbf{gradient flow} if $\bA=\ggrad(\br)$ for some $\br$, or equivalently if $\bA_{ij} = \br_i - \br_j$ for all $i,j$.
The \textbf{combinatorial divergence} of a flow is the $n$-vector $\dv(\bA) := \frac{1}{n}\bA\cdot \bO$. The divergence measures the contribution to the flow of each vertex, considered as a source. The \textbf{curl} of a flow is the three-tensor $\curl(\bA)_{ijk} = \bA_{ij} + \bA_{jk} - \bA_{ik}$. 

%%%%%%%%%%%%%%%%%%%%%%%%%%%%%%
\subsection{Functional Hodge theory}

We now extend the basic tools of combinatorial Hodge theory to the functional setting.

Let $W$ be a compact set with a probability measure. Given function $\phi:W\times W\rightarrow \bR$, let $\phi(\vt,\wt)$ prescribe the \textbf{flow} from $\vt$ to $\wt$. The opposite flow from $\wt$ to $\vt$ is $\phi(\wt,\vt) = -\phi(\vt,\wt)$, so $\phi:W\times W\rightarrow \bR$ is antisymmetric. The combinatorial setting above arises as the special case where the compact set $W$ is finite.

The \textbf{combinatorial gradient}\footnote{Our use of `gradient' and `divergence' does not coincide with the usual usage in multivariate calculus. This is forced on us by the terminological discrepancy between combinatorial Hodge theory and multivariate calculus. No confusion should result; we use $\ggrad(\bullet)$ for combinatorial gradients and $\grad$ for calculus gradients.} of a function $f:W\rightarrow \bR$ is the flow $\ggrad(f):W\times W\rightarrow \bR$ given by
\begin{equation}
    \ggrad(f)(\vt,\wt) := f(\vt) - f(\wt).
\end{equation}
The \textbf{divergence} of a flow is the function $\dv(\phi):W\rightarrow \bR$ given by
\begin{equation}
    \dv(\phi)(\vt) := \int_{W} \phi(\vt, \wt)\cdot d\wt.
\end{equation} 
The \textbf{curl} of flow $\phi$ is $\curl(\phi):W\times W\times W\rightarrow \bR$ given by
\begin{equation}
    \curl(\phi)(\bu, \vt, \wt) 
    := \phi(\bu,\vt) + \phi(\vt,\wt) - \phi(\bu,\wt).
\end{equation}

The following  proposition lays the groundwork for the Hodge decomposition, proved in the next section. 

\begin{lem}\label{lem:hodge}
	Let $\cF = \{f:W\rightarrow \bR\,:\,\int_Wf(\wt)d\wt = 0\}$ be the vector space of functions with zero expectation. Then
    %\dba{(when) is there a Harmonic component?}
    \textbf{(i)} $\dv\circ \ggrad(f) = f$; 
    \textbf{(ii}) $\dv(\phi)\in\cF$ for all flows $\phi$; and
    \textbf{(iii)} $\curl\circ \ggrad(f) \equiv0$ for all $f\in\cF$. 
\end{lem}

\begin{proof}
    \textbf{(i)}
    Observe that 
    \begin{align}
        \dv\circ \ggrad(f) & = \int_W \Big(f(\vt)-f(\wt)\Big)\cdot d\wt\\
        & = f(\vt)\int_W d\wt - \int_W f(\wt)d\wt\\
        & = f(\vt) - 0
    \end{align}
    where $\int_W d\wt=1$ because $d\wt$ is a probability measure and $\int_W f(\wt)d\wt=0$ by assumption.
    
    \textbf{(ii)}
    Direct computation  obtains that 
    \begin{align}
        \int_W \dv(\phi)(\vt)d\vt 
        & = \int_{W\times W} \phi(\vt,\wt)\cdot d\vt\cdot d\wt = 0
    \end{align}
    by antisymmetry of $\phi$.
    
    \textbf{(iii)}
    Direct computation shows that $\curl\circ \ggrad(f)$ is
    \begin{gather}
        \ggrad(f)(\bu,\vt) + \ggrad(f)(\vt,\wt) - \ggrad(f)(\bu,\wt)\\
        = f(\bu) - f(\vt) + f(\vt) - f(\wt) -f(\bu) + f(\wt) \equiv0
    \end{gather}
    as required.
\end{proof}

A basic result in combinatorial Hodge theory is the Hodge decomposition, which decomposes the space of flows into gradient-flows and curl-free flows \citep{jiang:11, reval:18}. The result is analogous to the Helmholtz decomposition in electrodynamics \citep{dgm:18}. Here, we prove a variant of the Hodge decomposition in the functional setting. 

\begin{thm*}[Hodge decomposition]
	The vector space of games admits an orthogonal decomposition
    \begin{equation}
        \{\text{transitive games}\} \oplus \{\text{cyclic games}\}
        = \im (\ggrad) \oplus \ker(\dv)
    \end{equation}
    with respect to the inner product on games $\langle\phi, \psi\rangle := \int_W \phi(\vt, \wt)\psi(\vt',\wt')\cdot d\vt d\wt d\vt' d\wt'$.
\end{thm*}
\begin{proof}
    First we show that $\im(\ggrad)$ and $\ker(\dv)$ are orthogonal. Suppose $\phi=\ggrad(f)$ and $\dv(\psi) = 0$. Then
    \begin{align}
        \langle \phi,\psi\rangle 
        = & \int \Big(f(\vt) - f(\wt)\Big)\psi(\vt',\wt')d\vt d\vt' d\wt d\wt'\\
        = & \int f(\vt)\psi(\vt',\wt')d\vt d\vt'  d\wt' \\
        & - \int f(\wt)\psi(\vt',\wt') d\vt' d\wt d\wt' \\
        \equiv &\,\, 0.
    \end{align}
    Second, observe that any flow $\phi$ can be written as 
    \begin{equation}
        \phi = \underbrace{\ggrad\circ\dv(\phi)}_{\ggrad(f)} 
        + \underbrace{\big(\phi - \ggrad\circ \dv(\phi)\big)}_{\psi}
    \end{equation}
    where $f= \dv(\phi)$ and $\psi := \phi - \ggrad\circ \dv(\phi)$ satisfies 
    \begin{align}
        \dv(\psi) 
        & =  \dv(\phi) - \dv\circ\ggrad (\dv \phi) \\
        & = \dv(\phi) - \dv \phi = 0,
    \end{align}
    because $\dv\circ\ggrad (\dv \phi)=\dv(\phi)$ by lemma~\ref{lem:hodge} \textbf{(i)} and \textbf{(ii)}.
\end{proof}

%%%%%%%%%%%%%%%%%%%%%%%%%%%%%%%%%%%%%%%%%%%
\section{Proofs of propositions}
\label{s:proofs}

\textbf{Notation.}
The dot product is $\vt^\intercal\wt = \langle\vt,\wt\rangle$. 
Subscripts can indicate the dimensions of vectors and matrices, $\vt_{n\times1}$ or $\bA_{m\times n}$, or their entries,  $\vt_i$ or $\bA_{ij}$; no confusion should result. The unit vector with a 1 in coordinate $i$ is $\be_i$.

%%%%%%%%%%%%%%%%%%%%%%%%%%%%%%%%%%%%%%
\paragraph{Proof of proposition~\ref{prop:invariance}.}
Before proving  the proposition, we first tighten up the terminology.

Two agents $\vt$ and $\vt'$ are behaviorally identical if $\phi(\vt,\wt) = \phi(\vt', \wt)$ for all $\wt\in W$. Given population $\pop$, two agents are empirically indistinguishable if $\phi(\vt,\wt) = \phi(\vt',\wt)$ for all  $\wt\in\pop$.

\begin{defn}
    \label{def:redundant}
    Population $\popq$ is redundant relative to population $\pop$ if every agent in $\popq$ is empirically indistinguishable from a convex mixture of agents in $\pop$.
    
    More formally, for all $\vt$ in $\popq$ there is an agent $\vt' \in \cG_\pop$ such that $\phi(\vt,\wt) = \phi(\vt',\wt)$ for all $\wt\in\pop$.
\end{defn}

Next, we need a notion of equivalence that applies to objects of different dimensions. 

\begin{defn}
    \label{def:equivalent}
    Two polytopes $\cP\subset \bR^m$ and $\cQ\subset \bR^{m+n}$ are \textbf{equivalent} if there is an orthogonal projection $\bpi_{(m+n)\times m}$ with left inverse $\bxi_{m\times (m+n)}$ (i.e. satisfying $\bxi\cdot \bpi = \bI_{m\times m}$) such that
    \begin{equation}
        \cP = \cQ\cdot \bpi
        \quad\text{and}\quad
        \cQ = \cP\cdot \bxi.
    \end{equation}
\end{defn}

\begin{prop*}
    Suppose $\popq\supset\pop$ is redundant relative to $\pop$. Then the gamescapes $\cG_\pop$ and $\cG_\popq$ are equivalent.
\end{prop*}

\begin{proof}
    Suppose that $\pop$ has $m$ elements  and $\popq$ has $m+n$ elements. We assume, without loss of generality, that the elements of $\popq\supset \pop$ are ordered such that the first $m$ coincide with $\pop$. The $(m+n)\times (m+n)$ evaluation matrix $\bA_\popq$ thus contains the $(m\times m)$ evaluation matrix $\bA_\pop$ as a `top left' sub-block.
    
    By assumption (since the additional $n$ agents  in $\popq$ are convex combinations of agents in $\pop$), evaluation matrix $\bA_\popq$ has the block form
    \begin{equation}
        \bA_\popq  = \left[\begin{matrix}
            \bA_\pop &  -\bA_\pop^\intercal \bM^\intercal \\
            \bM\bA_\pop & 
            - \bM\bA_\pop^\intercal\bM^\intercal
        \end{matrix}\right]
         = \left[\begin{matrix}
            \bI_{m\times m} \\
            \bM_{n\times m}
        \end{matrix}\right]\cdot
        \bA_\pop\cdot 
        \left[\begin{matrix}
            \bI_{m\times m} \\ \bM_{n\times m}
        \end{matrix}\right]^\intercal
    \end{equation}
    where $\bM$ specifies the convex combinations. It follows that $\cG_\popq$ is generated by convex mixtures of the first $m$ rows of $\bA_\popq$:
    \begin{equation}
        \left[\begin{matrix}
            \bA_\pop & -\bA_\pop^\intercal \bM^\intercal
        \end{matrix}\right].
    \end{equation}
    Now, let $\bA := \bA_\pop$
    \begin{equation}
        \bB := \left[\begin{matrix}
            \bA_\pop & -\bA_\pop^\intercal \bM^\intercal
        \end{matrix}\right] 
        = \bA_\pop\cdot 
        \left[\begin{matrix}
            \bI_{m\times m} \\ \bM
        \end{matrix}\right]^\intercal,
    \end{equation}
    and let 
    \begin{equation}
        \bxi := \left[\begin{matrix}
            \bI_{m\times m} & \bM^\intercal
        \end{matrix}\right]
        \quad\text{and}\quad
        \bpi := \left[\begin{matrix}
            \bI_{m\times m} \\ \bZ
        \end{matrix}\right].
    \end{equation}
    Then $\bA$ and $\bB$ are equivalent under definition~\ref{def:equivalent}. It follows that the gamescape of $\bA$ is the orthgonal projection under $\bpi$ of the gamescape of $\bB$, and that $\bB$ can be recovered from $\bA$ by applying $\bxi$, the right inverse of $\bpi$.
\end{proof}

%%%%%%%%%%%%%%%%%%%%%%%%%%%%%%%%%%%%%%
\paragraph{Proof of proposition~\ref{prop:schur}.}

\begin{prop*}
    The $\egs$ of population $\pop$ can be represented in $\bR^r$, where $r=\rank(\bA_\pop)$.
\end{prop*}

\begin{proof}
    The rank of an antisymmetric matrix $\bA$ is even, so let $r=2k$. The Schur decomposition, see \citep{reval:18}, factorizes an antisymmetric matrix as 
    \begin{equation}
        \bA_{n\times n} = \Wt_{n\times 2k} \cdot \bJ_{2k\times 2k} \cdot \Wt^\intercal_{2k\times n}
    \end{equation}
    where $\bJ_{2k\times 2k} = \sum_{i=1}^k \lambda_i\cdot (\be_{2i}\be_{2i-1}^\intercal - \be_{2i-1}\be_{2i}^\intercal)$ with $\lambda_i>0$ and the rows of $\Wt_{n\times 2k}$ are orthogonal.
    
    Let $\bG := \Wt_{n\times 2k} \cdot \bJ_{2k\times 2k}$. We claim the empirical gamescape, given by convex mixtures of rows of $\bA$, is isomorphic to the polytope given by convex mixtures of rows of $\bG$, which lives in $\bR^{2k}=\bR^r$. To see this, observe that $\bA=\bG\cdot \Wt^\intercal$ and the columns of $\Wt^\intercal$ are orthogonal. 
\end{proof}

%%%%%%%%%%%%%%%%%%%%%%%%%%%%%%%%%%%%%%
\paragraph{Proof of proposition~\ref{prop:nash}.}
\begin{prop*}
    Geometrically, the empirical Nash equilibria are 
    \begin{equation}
	    \cN_\pop 
	    %= \{\x \in \cG_\pop: \x\succeq\bZ\} 
	    = \{\bp \text{ distribution}: \bp^\intercal\bA_\pop\succeq\bZ\}.
    \end{equation}
\end{prop*}

\begin{proof}
    Recall that the Nash equilibria, of the column-player, in a two-player zero-sum game specified by matrix $\bA$, are the solutions to the linear program:
    \begin{align}
        & \max_{v\in\bR} v \\
        \text{s.t.}\quad & \bp^\intercal\bA \succeq v\cdot \bO\\
        & \bp\succeq \bZ \text{ and }
         \bp^\intercal\bO = 1
    \end{align}
    where the resulting $v^*$ is the value of the game. Since $\bA$ is antisymmetric, the value of the game is zero and the result follows.
\end{proof}

%%%%%%%%%%%%%%%%%%%%%%%%%%%%%%%%%%%%%%
\paragraph{Proof of proposition~\ref{prop:perf}.}

\begin{prop*}
	\textbf{(i)} 
	Performance $\perf$ is independent of the choice of Nash equilibrium. 
	\textbf{(ii)} 
	If $\phi$ is monotonic then performance compares the best agents in each population
	\begin{equation}
    	\perf(\pop,\popq) = 
    	\max_{\vt\in\pop} f(\vt) - \max_{\wt\in\popq} f(\wt).
	\end{equation}
	\textbf{(iii)} 
	If $\hull(\pop) \subset \hull(\popq)$ then $\perf(\pop,\popq)\leq0$ and $\perf(\pop, \popr)\leq \perf(\popq,\popr)$ for \textbf{any} population $\popr$.
\end{prop*}

\begin{proof}
    \textbf{(i)}
    The value $\bp^\intercal\bA\bq$ of a zero-sum two-player matrix game is independent of the choice of Nash equilibrium $(\bp,\bq)$.
    
    \textbf{(ii)}
    If $\phi=\sigma(\ggrad(f))$ then 
    \begin{equation}
        \bA_\pop = \Big(\begin{matrix}
            \sigma\big(f(\wt_i) - f(\wt_j)\big)
        \end{matrix}\Big)_{i,j=1}^n 
    \end{equation}
    and the Nash equilibrium for each player is to concentrate its mass on the set 
    \begin{equation}
        \argmax_{i\in[n]}f(\wt_i).
    \end{equation}
    The result follows immediately.
    
    \textbf{(iii)}
    If $\hull(\pop) \subset \hull(\popq)$ then the Nash $\bp$ on $\pop$ of the row player can be reconstructed by the column player (since every mixture within $\pop$ is also a mixture of agents in $\popq$. Thus, $\perf(\pop,\popq)$ is at most zero.
    
    Similarly, if $\hull(\pop) \subset \hull(\popq)$ then every mixture available to $\pop$ is also available to $\popq$, so $\pop$'s performance against any other population $\popr$ will be the same or worse than the performance of $\popq$.
\end{proof}

%%%%%%%%%%%%%%%%%%%%%%%%%%%%%%%%%%%%%%
\paragraph{Proof of proposition~\ref{prop:oracle_guarantee}.}

\begin{prop*}
    If $\bp$ is a Nash equilibrium on $\bA_\pop$ and $\sum_i p_i \phi_{\wt_i}(\vt)>0$, then adding $\vt$ to $\pop$ strictly enlarges the gamescape: $\cG_\pop \subsetneq \cG_{\pop\cup\{\vt\}}$.
\end{prop*}

\begin{proof}
    We are required to show that $\vt$ is not a convex combination of agents in $\pop$. This follows by contradiction: if $\vt$ were a convex combination of agents in $\pop$ then it would either tie or lose to the Nash distribution -- whereas $\vt$ beats the Nash.
\end{proof}

%%%%%%%%%%%%%%%%%%%%%%%%%%%%%%%%%%%%%%%%%%%%%%%%%
\paragraph{Proof of proposition~\ref{prop:threefold}}

\begin{prop*}
    The objective constructed by rectified Nash response is effective diversity, definition~\ref{def:diverse}.
\end{prop*}
\begin{proof}
    Rewrite 
    \begin{align}
        \diverse(\pop) 
        & = \sum_{i,j=1}^n \floor*{\phi(\wt_i,\wt_j)}_+ \cdot p_ip_j \\
        & = \sum_i p_i \sum_j p_j  \floor*{\phi_{\wt_j}(\wt_i)}_+ 
        =: \sum_i p_i\cdot h_i(\wt_i)
    \end{align}
    where $h_i(\wt_i) := \sum_j p_j \floor*{\phi_{\wt_j}(\wt_i)}_+$. Recall that the rectified Nash response objectives are 
    \begin{equation}
        \expec_{\wt\sim\bp_t}\left[
            \floor*{\phi_\wt(\vt_t)}_+
        \right].
    \end{equation}
    Finally, notice that $h_j(\vt_t) = \expec_{\wt\sim\bp_t}\left[\floor*{\phi_\wt(\vt_t)}_+\right]$ as required.
\end{proof}

%%%%%%%%%%%%%%%%%%%%%%%%%%%%%%%%%%%%%%%%%%%%%%%%%%%%%%%%%%
\section{Rectified Nash response and a reduction to rock-paper-scissors}
\label{s:reduction}

In this section, we present some basic analysis of the geometry of the gamescape and its relation to Nash equilibria.

%%%%%%%%%%%%%%%%%%%%%%%%%%%%%%%%%%%%%%%%%%%%
\subsection{Nash reweighting of evaluation matrices}
Let $\bA$ be an antisymmetric matrix with Nash equilibrium $\bp$ (not necessarily unique).

\begin{lem}\label{lem:nash_reweight}
    The Nash reweighted matrix $\bp\odot \bA \odot \bp$ defined by setting its $ij^\text{th}$ entry to
    \begin{equation}
        (\bp\odot \bA \odot \bp)_{ij} := 
        A_{ij}\cdot p_i\cdot p_j
    \end{equation}
    is \textbf{(i)} antisymmetric and \textbf{(ii)} satisfies $(\bp\odot \bA \odot \bp)\cdot \bO = \bZ$ and $\bO^\intercal \cdot (\bp\odot \bA \odot \bp) = \bZ^\intercal$. That is, all its rows and columns sum to zero.
\end{lem}

\begin{proof}
    \textbf{(i)}
    Antisymmetry holds since $A_{ij} p_i p_j + A_{ji} p_j p_i = A_{ij} p_i p_j - A_{ij} p_i p_j = 0$ by antisymmetry of $\bA$.
    
    \textbf{(ii)}
    We show that all the entries of the vector  $(\bp\odot \bA \odot \bp)\cdot \bO$ are zero, by showing that they are all nonnegative and that they sum to zero. 
    
    Direct computation obtains that the $i^\text{th}$ entry of $(\bp\odot \bA \odot \bp)\cdot \bO$ is $p_i \sum_j A_{ij}p_j = p_i\cdot (\bA\bp)[i]$. Recall that $\bA\bp\succeq \bZ$ since $\bp$ is a Nash equilibrium and the value of the game is zero (since $\bA$ is antisymmetric). Thus, all the entries of $(\bp\odot \bA \odot \bp)\cdot \bO$ are nonnegative. Finally, $\sum_{ij}A_{ij}p_ip_j = \bp^\intercal\bA\bp=0$ because $\bA$ is antisymmetric.
\end{proof}

%%%%%%%%%%%%%%%%%%%%%%%%%%%%%%%%%%%%%%%%%%%%
\subsection{Diversity as a matrix norm}
\label{s:div_norm}

We show that effective diversity, definition~\ref{def:diverse}, arises as a particular matrix norm. Specifically, the diversity of a population can be rewritten as a matrix norm $\ell_{p,q}$ for $p=q=1$. An interesting direction, beyond the scope of this paper, is to investigate the properties and performance of alternate diversity measures based on more common matrix norms such as the Frobenius and spectral norms.

Recall that
\begin{equation}
    \|\bA\|_{p,q} := \left(\sum_i \left(\sum_j|a_{ij}|^p\right)^{q/p}\right)^{1/q}
\end{equation}
and so, in particular, 
\begin{equation}
    \|\bA\|_{1,1} := \sum_{ij} |a_{ij}|.
\end{equation}
\begin{lem}
    The diversity of population $\pop$ is 
    \begin{equation}
        d(\pop) = \frac{1}{2}\left\|\bp\odot\bA_\pop\odot \bp\right\|_{1,1}
    \end{equation}
\end{lem}
\begin{proof}
    Notice that
    \begin{align}
        \|\bp\odot\bA_\pop\odot \bp\|_{1,1}
        & = \sum_{ij} |\phi(\wt_i,\wt_j)|\cdot p_i p_j \\
        & =\sum_{ij} 2\floor*{\phi(\wt_i,\wt_j)}_+\cdot p_i p_j 
    \end{align}
    where the second equality follows by antisymmetry of $\phi$.
\end{proof}

%%%%%%%%%%%%%%%%%%%%%%%%%%%%%%%%%%%%%%%%%%%%
\subsection{Reduction to the rock-paper-scissors}

Let $\bA$ be an antisymmetric matrix with Nash equilibrium $\bp$ (not necessarily unique). Suppose that $\bp$ has nonzero mass in at least three\footnote{If the Nash equilibrium has support on two agents, say $\vt$ and $\wt$, then necessarily $\phi(\vt,\wt)=0$.} entries. Pick a row of $\bA$ with positive mass under the Nash -- which we can assume to be the first row without loss of generality.

\paragraph{Rock-paper-scissors reduction.}
Construct weight vectors:
\begin{align}
    \bp_\rock[i] & := 
    \begin{cases}
        \bp[1] & \text{if } i=1\\
        0 & \text{else}
    \end{cases} 
    \\
    \bp_\paper[i] & := 
    \begin{cases}
        \bp[i] & \text{if } \phi(\wt_1,\wt_i)>0\\
        0 & \text{else}
    \end{cases} 
    \\
    \bp_\scissors[i] & :=
    \begin{cases}
        \bp[i] &  \text{if } \phi(\wt_1,\wt_i)\leq0 \text{ and }  i\neq 1\\
        0 & \text{else}
    \end{cases} 
\end{align}
and corresponding mixture agents:
\begin{align}
    \rock & := \sum_{i\in[n]} \bp_\rock[i]\cdot \wt_i \\
    \paper & := \sum_{i\in[n]} \bp_\paper[i]\cdot \wt_i  \\
    \scissors & := \sum_{i\in[n]} \bp_\scissors[i]\cdot \wt_i 
\end{align}
Mixtures of agents are evaluated in the obvious way, for example
\begin{equation}
    \phi(\rock, \scissors) = \sum_{i,j\in[n]} \bp_\rock[i]\cdot \bp_\scissors[j]\cdot \phi(\wt_i,\wt_j).
\end{equation}

\begin{prop}[the reduction creates a rock-paper-scissors meta-game]
    The evaluation matrix for the population $\three = \{\rock, \paper, \scissors\}$ has the form
    \begin{equation}
        \bA_{\three} = \left(\begin{matrix}
            0 & \alpha & -\alpha \\
            -\alpha & 0 &\alpha \\
            \alpha & -\alpha & 0
        \end{matrix}\right)
    \end{equation}
    for some $\alpha = \phi(\rock, \scissors)>0$.
\end{prop}

\begin{proof}
    In general, a $(3\times 3)$ antisymmetric matrix must have the form
    \begin{equation}
        \left(\begin{matrix}
            0 & \alpha & -\beta \\
            -\alpha & 0 & \gamma \\
            \beta & -\gamma & 0
        \end{matrix}\right)
    \end{equation}
    for some $\alpha, \beta, \gamma \in \bR$. The content of the proposition is \textbf{(i)} that $\alpha>0$ and \textbf{(ii)} that $\alpha=\beta=\gamma$. 
    
    \textbf{(i)}
    It is clear that $\alpha>0$ since $\alpha:=\phi(\rock, \scissors)>0$ by construction. 
    
    \textbf{(ii)}
    The matrix $\bA_\three$ can alternatively be constructed as follows. Take the Nash reweighted matrix $\bp\odot \bA\odot \bp$ and group the rows and columns into rock, paper and scissor blocks, to obtain a matrix with $3\times 3$ blocks. Sum the entries within each of the blocks to obtain the $3\times 3$ matrix $\bA_\three$. Lemma~\ref{lem:nash_reweight} then implies that $\bA_\three\cdot\bO = \bZ$, which in turn implies $\alpha=\beta=\gamma$.
\end{proof}

%%%%%%%%%%%%%%%%%%%%%%%%%%%%%%%%%%%%%%%%%%%%%%%%%%%%%%%%%%%%%%
\section{Examples of functional-form games}
\label{s:ffg-egs}

Go, Chess, Starcraft, 1v1 DOTA, poker and other two-player zero-sum games are functional-form games when played by parametrized agents. The function $\phi$ is difficult to compute, especially when agents are stochastic. Nevertheless, it can be estimated and approximate best-responses to agents or mixtures of agents can be estimated by reinforcement learning algorithms.

In this section, we present a variety of explicit, mathematical constructions of $\ffg$'s.

%%%%%%%%%%%%%%%%%%%%%%%%%%%%%%%%%%%%%%%%
\subsection{Symplectic games}

Let $W\subset\bR^{2d}$ and define $\phi:W\times W\rightarrow \bR$ as
\begin{equation}
    \label{eq:symplectic}
	\phi(\vt, \wt) = \vt^\intercal\cdot \bOmega_{2d}\cdot \wt
\end{equation}
where $\bOmega_{2d}$ is the antisymmetric matrix
\begin{equation}
	\bOmega_{2d} = \left(\begin{matrix}
		0 & -1 & \cdots & 0 & 0 \\
		1 & 0 &  \cdots & 0 & 0\\
		\vdots & \vdots & \ddots & \vdots & \vdots\\
		0 & 0 &  \cdots &0& -1\\
		0 & 0 &  \cdots &1& 0
	\end{matrix}\right) 
	= \sum_{i=1}^d \be_{2i}\be_{2i-1}^\intercal - \be_{2i-1}\be_{2i}^\intercal 
\end{equation}

\paragraph{Transitive special case.}
If $d=2$ and $\vt$ has the form $\vt=(\alpha, v)$ then
\begin{equation}
    \phi(\vt, \wt) = \vt^\intercal\cdot \bOmega_{2}\cdot \wt = \alpha(v-w)
\end{equation}
which gives a transitive game with $f(\vt) = \alpha\cdot v$. 

\paragraph{Cyclic special case.}
A high dimensional cyclic game is obtained by setting 
\begin{equation}
    W = \{\vt \in\bR^{2d}\,:\,\|\vt\|_2^2<k\}
\end{equation}
for some $k>0$ and setting $\phi$ per eq.~\eqref{eq:symplectic}.

\paragraph{Mixed transitive and cyclic.}
The following symplectic game incorporates both transitive and cyclic aspects. Suppose $d>2$ and
\begin{equation}
    W=\left\{\vt\in\bR^{2d}\,:\,v_1=1, |v_2|\leq 100,  \text{ and }\|\vt_{3:2d}\|_2^2\leq 1\right\}.
\end{equation}
In other words, $\vt\in W$ has the form $\vt=(1,v_1,v_2,\ldots,v_{2d})$ where $v_1=1$, $v_2\in[-100,100]$ and $(v_3,\ldots,v_{2d})$ lies in the unit ball in $\bR^{2d-2}$.

Then $\phi:W\times W\rightarrow \bR$ given by 
\begin{equation}
	\phi(\vt,\wt) 
	= \vt^\intercal\cdot \bOmega_{2d}\cdot \wt
\end{equation}
as above specifies a game that resembles a cigar -- with both transitive (given by the first two coordinates) and cyclic (given by the last $2d-2$ coordinates) structure.

%%%%%%%%%%%%%%%%%%%%%%%%%%%%%%%
\paragraph{Long cycles.}
We provide a little more detail on example~\ref{eg:long}. Suppose $n$ agents form a long cycle: $\pop = \{\vt_1 \xrightarrow{\text{beats}}\vt_2 \rightarrow\cdots \rightarrow\vt_n \xrightarrow{\text{beats}} \vt_1\}$. Then
\begin{equation}
    \bA_\pop = \sum_{i=1}^n \be_{\floor{i}_n}\be_{\floor{i+1}_n}^\intercal - \be_{\floor{i-2}_n}\be_{\floor{i-1}_n}^\intercal,
\end{equation}
where ${\floor{i}_n}$ is $i$ modulo $n$. Direct computation shows $\text{rank}(\bA_\pop)$ is $n-2$ if $n$ is even and $n-1$ if $n$ is odd.

%%%%%%%%%%%%%%%%%%%%%%%%%%%%%%%
\paragraph{Deformed symplectic games.}
Let $\ft_i:\bR^n\rightarrow \bR^{2d}$ be differentiable multivariate functions and suppose $g_i:\bR\rightarrow \bR$ are odd for all $i$, i.e. satisfy $g_i(-x) = -g_i(x)$. A large class of differentiable $\ffg$'s are of the form 
\begin{equation}
    \label{eq:nonlinear}
	\phi(\vt,\wt) := \sum_i g_i\Big(\ft_i(\vt)^\intercal
	\cdot\bOmega_{2d}\cdot \ft_i(\wt)
	\Big).
\end{equation}
where $\phi:W\times W\rightarrow \bR$ for some $W\subset\bR^n$. The class is further enlarged by considering linear combinations of functions of the form in eq.~\eqref{eq:nonlinear}.

%%%%%%%%%%%%%%%%%%%%%%%%%%%%%%%
\paragraph{A strange example.}
Let 
\begin{equation}
    W = \left\{\vt\in\bR^{2d} \,:\,\|\vt\|_1 = 1\right\}
\end{equation}
equipped with the uniform probability measure and suppose $\ft:\bR^{2d}\rightarrow\bR^{2d}$ is given by
\begin{equation}
    \ft(\vt) = \left(v_1^{99}, \ldots, v_{2d}^{99}\right)
\end{equation}
It follows that $\ft(W)\subset W$, and moreover that the image of $W$ under $\ft$ is concentrated on extremely sparse vectors, with most entries near zero. Finally consider the $\ffg$ given by
\begin{equation}
	\phi(\vt,\wt) := \ft_i(\vt)^\intercal
	\cdot\bOmega_{2d}\cdot \ft_i(\wt).
\end{equation}

%%%%%%%%%%%%%%%%%%%%%%%%%%%%%%%
\subsection{Variants of Colonel Blotto}

The Colonel Blotto game was introduced in \citet{borel:21}. The game concerns two players that have two distribute limited resources over several battlefields.

%%%%%%%%%%%%%%%%%%%%%%%%%%%%%%%
\paragraph{Continuous Lotto}

Fix $a>0$. Continuous Lotto is a two-player symmetric zero-sum game introduced in \citet{hart:08}. An `agent' is a real-valued non-negative random variable $X$ satisfying $\expec[X] = a$. Given agents $X$ and $Y$, define
\begin{equation}
	\phi(X,Y) := P(X > Y) - P(X < Y)
\end{equation}
where
\begin{equation}
	P(X > Y) = \int F_X(y) f_Y(y) dy
\end{equation}
for $F_X(y) = \int_{x=-\infty}^y f_X(x) dx$.

Picking different parametrized families of random variables, for example mixtures of $k$ Diracs or mixtures of $k$ Gaussians, yields a variety of functional form games.

%%%%%%%%%%%%%%%%%%%%%%%%%%%%%%%
\paragraph{Differentiable Lotto.} 
Many variants of the distance game, described in the main text, can be constructed, to obtain more examples of differentiable zero-sum symmetric functional-form games.

%%%%%%%%%%%%%%%%%%%%%%%%%%%%%%%%
\section{Details of 2D embedding procedure}
\label{s:embeddings}

In Figure~\ref{f:two_leaderboards} and \ref{f:embeddings} we used payoff tables generated in the following way. First, we sampled $N \times N$ matrix $E$ of independent samples from Normal distribution with $\sigma = 0.02$, and then transformed it to an antisymmetric matrix by putting $P_\mathrm{random} = (E - E^\mathrm{T})/2$ which became the Random payoff. Next we created a purely transitive payoff matrix and added $P_\mathrm{random}$ to it, creating Almost Transitive payoff $P_\mathrm{trans}$ and analogously for Almost Cyclic $P_\mathrm{cyclic}$. The Mixed payoff is a convex combination of the above, namely $P_\mathrm{mixed} = 0.65 P_\mathrm{trans}  + 0.35 P_\mathrm{cyc}$. Finally the almost monotonic payoff has been obtained by using $\mathrm{sign}$ function applied point-wise to each element of $P_\mathrm{tran}$ before the noise was added, so $P_\mathrm{monotonic} = \mathrm{sign}(P_\mathrm{trans} - E) + E$.

In order to obtain embeddings we used following methods for each of payoffs $P$:
\begin{itemize}
    \item \textbf{Schur}, Schur decomposition of $P$ has been computed $P = QUQ^{-1}$, and then first two dimencions of $U$ were used as an embedding.
    \item \textbf{PCA}, Principal Component Analysis was applied to N points, where attributes of $i^\mathrm{th}$ datapoint were taken to be an $i$th row of the $P$ and 2 most informative dimensions were used.
    \item \textbf{SVD}, Singular-value Decomposition $P=U\Sigma V^*$ was used analogously to the PCA, by keeping dimensions with two highest singular values. 
    \item \textbf{tSNE}, t-Distributed Stochastic Neighbor Embedding, applied analogously to PCA was used with perplexity 9 (selected empirically, with various choices leading to significantly different embeddings, none of which seemed to capture gamescape structure in a more meaningful way).
\end{itemize}

\end{document}